\documentclass[sigconf]{acmart}

\AtBeginDocument{%
  \providecommand\BibTeX{{%
    \normalfont B\kern-0.5em{\scshape i\kern-0.25em b}\kern-0.8em\TeX}}}

\copyrightyear{2022}
\acmYear{2022}
\setcopyright{acmcopyright}\acmConference[KDD '22]{Proceedings of the 28th ACM SIGKDD Conference on Knowledge Discovery and Data Mining}{August 14--18, 2022}{Washington, DC, USA}
\acmBooktitle{Proceedings of the 28th ACM SIGKDD Conference on Knowledge Discovery and Data Mining (KDD '22), August 14--18, 2022, Washington, DC, USA}
\acmPrice{15.00}
\acmDOI{10.1145/3534678.3539429}
\acmISBN{978-1-4503-9385-0/22/08}

\usepackage{subfig}

\usepackage[ruled,vlined,linesnumbered,noresetcount]{algorithm2e}
\usepackage{threeparttable,color}
\usepackage{booktabs}
\usepackage{multirow}
\usepackage{balance}
\usepackage{bbm,bm}
\usepackage{setspace}
\usepackage{algorithmic}
\usepackage{graphicx}
\usepackage{paralist}

\usepackage{amsmath,amsfonts,enumitem}

\definecolor{green}{rgb}{0.05, 0.5, 0.06}

\newcommand{\method}{\textit{DosCond}\xspace}
\newtheorem{theorem}{Theorem}

\theoremstyle{remark}
\newtheorem{remark}[theorem]{Remark}

\begin{document}
\title{Condensing  Graphs via One-Step Gradient Matching}






\author{Wei Jin}\authornote{Work done while author was on internship at Amazon.}
\affiliation{%
  \institution{Michigan State University}
}
\email{jinwei2@msu.edu}

\author{Xianfeng Tang}
\affiliation{%
  \institution{Amazon}
}
\email{xianft@amazon.com}

\author{Haoming Jiang}
\affiliation{%
  \institution{Amazon}
}
\email{jhaoming@amazon.com}

\author{Zheng Li}
\affiliation{%
  \institution{Amazon}
}
\email{amzzhe@amazon.com}

\author{Danqing Zhang}
\affiliation{%
  \institution{Amazon}
}
\email{danqinz@amazon.com}

\author{Jiliang Tang}
\affiliation{%
  \institution{ Michigan State University}
}
\email{tangjili@msu.edu}

\author{Bing Yin}
\affiliation{%
  \institution{Amazon}
}
\email{alexbyin@amazon.com}

\renewcommand{\shortauthors}{Wei Jin et al.}

\begin{abstract}
As training deep learning models on large dataset takes a lot of time and resources, it is desired to construct a small synthetic dataset with which we can train deep learning models sufficiently. There are recent works that have explored solutions on condensing image datasets through complex bi-level optimization. For instance, dataset condensation (DC) matches network gradients w.r.t. large-real data and small-synthetic data, where the network weights are optimized for multiple steps at each outer iteration. However, existing approaches have their inherent limitations: (1) they are not directly applicable to graphs where the data is discrete; and (2) the condensation process is computationally expensive due to the involved nested optimization. To bridge the gap, we investigate efficient dataset condensation tailored for graph datasets where we model the discrete graph structure as a probabilistic model. We further propose a one-step gradient matching scheme, which performs gradient matching for only one single step without training the network weights. Our theoretical analysis shows this strategy can generate synthetic graphs that lead to lower classification loss on real graphs.  Extensive experiments on various graph datasets demonstrate the effectiveness and efficiency of the proposed method. In particular, we are able to reduce the dataset size by $90$\% while approximating up to $98$\% of the original performance and our method is significantly faster than multi-step gradient matching (e.g. $15$× in CIFAR10 for synthesizing $500$ graphs). {Code  is  available at \url{https://github.com/amazon-research/DosCond}.} 


\end{abstract}

\keywords{Data-efficient Learning; Graph Generation; Graph Neural Networks}

\ccsdesc[500]{Computing methodologies~Neural networks}

\maketitle

\section{Introduction}
Graph-structured data plays a key role in various real-world applications. For example, by exploiting graph structural information, we can predict the chemical property of a given molecular graph~\cite{ying2018hierarchical}, detect fraud activities in a financial transaction graph~\cite{card-fraud2019}, or recommend new friends to users in a social network~\cite{fan2019graph}. Due to its prevalence, graph neural networks (GNNs)~\cite{kipf2016semi,gat,battaglia2018relational,wu2019comprehensive-survey} have been developed to effectively extract meaningful patterns from graph data and thus tremendously facilitate computational tasks on graphs. Despite their effectiveness, GNNs are notoriously data-hungry like traditional deep neural networks: they usually require  massive datasets to learn powerful representations. Thus, training GNNs is often computationally expensive. Such cost even becomes prohibitive when we need to repeatedly train GNNs, e.g., in neural architecture search~\cite{liu2018darts} and continual learning~\cite{li2017learning}. 

One potential solution to alleviate the aforementioned issue is \textit{dataset condensation} or \textit{dataset distillation}. It targets at constructing a small-synthetic training set that can provide sufficient information to train neural networks~\cite{wang2018dataset,zhao2020dataset,zhao2021dataset,nguyen2021dataset,nguyen2021infinitely,cazenavette2022dataset,wang2022cafe}. In particular, one of the representative methods, DC~\cite{zhao2020dataset}, formulates the condensation goal as matching the gradients of the network parameters between small-synthetic and large-real training data.  
It has been demonstrated that such a solution can greatly reduce the training set size of image datasets without significantly sacrificing model performance. For example, using $100$ images generated by DC can achieve $97.4$\% test accuracy on MNIST compared with $99.6$\% on the original dataset ($60,000$ images). These condensed samples can significantly save space for storing datasets and speed up retraining neural networks in many critical applications, e.g., continual learning and neural architecture search. In spite of the recent advances in dataset distillation/condensation for images, limited attention has been paid on domains involving graph structures.

To bridge this gap, we investigate the problem of condensing graphs such that GNNs trained on condensed graphs can achieve comparable performance to those trained on the original dataset. 
However, directly applying existing solutions for dataset condensation~\cite{wang2018dataset,zhao2020dataset,zhao2021dataset,nguyen2021dataset} to graph domain faces some challenges. First, existing solutions have been designed for images where the data is continuous and they cannot output binary values to form the discrete graph structure. Thus, we need to develop a strategy that can handle the discrete nature of graphs. Second, they usually involve a complex bi-level problem that is computationally expensive to optimize: they require multiple iterations (inner iterations) of updating neural network parameters before updating the synthetic data for multiple iterations (outer iterations). It can be catastrophically inefficient for learning pairwise relations for nodes, of which the complexity is quadratic to the number of nodes. While one recent work targets at graph condensation for node classification~\cite{jin2022graph}, it does not overcome these challenges because it does not produce discrete graph structures and its condensation process is costly.

To address the aforementioned challenges, we propose an efficient condensation method for graphs, where we follow DC~\cite{zhao2020dataset} to match the gradients of GNNs between synthetic graphs and real graphs. In order to produce discrete values, we model the graph structure as a probabilistic graph model and optimize the discrete structures in a differentiable manner. Based on this formulation, we further propose a \textit{one-step gradient matching} strategy which only performs gradient matching for one single step. Consequently, the advantages of the proposed strategy are twofold. First, it significantly speeds up the condensation process while providing reasonable guidance for synthesizing condensed graphs. Second, it removes the burden of tuning hyper-parameters  such as the number of outer/inner iterations of the bi-level optimization as required by DC.  Furthermore, we demonstrate the effectiveness of the proposed one-step gradient matching strategy both theoretically and empirically. Our contributions can be summarized as follows:
\begin{compactenum}[1.]
\item We study a novel problem of learning discrete synthetic graphs for condensing graph datasets, where the discrete structure is captured via a graph probabilistic model that can be learned in a differentiable manner. 
\item  We propose a one-step gradient matching scheme that significantly accelerates the vanilla gradient matching process.
\item Theoretical analysis is provided to understand the rationality of the proposed one-step gradient matching. We show that  learning with one-step matching produces synthetic graphs that lead to a smaller classification loss on real graphs.
\item Extensive experiments have demonstrated the effectiveness and efficiency of the proposed method. Particularly, we are able to reduce the dataset size by $90$\% while approximating up to $98$\% of the original performance and our method is significantly faster than multi-step gradient matching (e.g. $15$× in CIFAR10 for synthesizing $500$ graphs).
\end{compactenum}

\section{The Proposed Framework}
Before detailing the framework, we first introduce the main notations used in this paper. We majorly focus on the graph classification task where the goal is to predict the labels of given graphs. Specifically, we denote a graph dataset as $\mathcal{T}=\{G_1, \ldots, G_N\}$ with ground-truth label set $\mathcal{Y}$. 
Each graph in $\mathcal{T}$ is associated with a discrete adjacency matrix and a node feature matrix. Let ${\bf A}_{(i)}$, ${\bf X}_{(i)}$ represent the adjacency matrix and  the feature matrix of $i$-th real graph, respectively. Similarly, we use $\mathcal{S}=\{G'_1, \ldots, G'_{N'}\}$ and $\mathcal{Y}'$ to indicate the synthetic graphs and their labels, respectively. Note that the number of synthetic graphs $N'$ is essentially much smaller than that of real graphs $N$. We use $d$ and $n$ to denote the number of feature dimensions the number of nodes in each synthetic graph, respectively\footnote{We set $n$ to the average number of nodes in  original dataset.}. Let $C$ denote the number of classes and $\ell$ denote the cross entropy loss. The goal of our work is to learn a set of synthetic graphs $\mathcal{S}$ such
that a GNN trained on $\mathcal{S}$ can achieve comparable performance to the one trained on the much larger dataset $\mathcal{T}$.

In the following subsections,  we first introduce how to apply the vanilla gradient matching to condensing graphs for graph classification  (Section~\ref{sec:gm}). However, it cannot generate discrete graph structure and is highly inefficient. To correspondingly address these two limitations, we discuss the approach to handling the discrete nature of graphs (Section~\ref{sec:discrete}) and propose an efficient solution, one-step gradient matching, which significantly accelerates the condensation process (Section~\ref{sec:one_step}). 

 
\subsection{Gradient Matching as the Objective}
\label{sec:gm}
Since we aim at learning synthetic graphs that are highly informative, one solution is to allow GNNs trained on synthetic graphs to imitate the training trajectory on the original large dataset. Dataset condensation~\cite{zhao2020dataset,zhao2021dataset} introduces a gradient matching scheme to achieve this goal. Concretely, it tries to reduce the difference of model gradients w.r.t. large-real data and small-synthetic data for model parameters at every training epoch. Hence, the model parameters trained on synthetic data will be close to these trained on real data at every training epoch. Let $\theta_t$ denote the network parameters at the $t$-th epoch and $f_{\theta_t}$ indicate the neural network parameterized by $\theta_t$. The condensation objective is expressed as:
\begin{align}
\min_{\mathcal{S}} \;  \sum_{t=0}^{T-1}  & D(\nabla_{\theta}\ell(f_{\theta_t}(\mathcal{S}),  \mathcal{Y}'), \nabla_{\theta}\ell(f_{\theta_t}(\mathcal{T}), \mathcal{Y})), \nonumber\\ 
& \quad\text{s.t.}  \; \theta_{t+1} =   \operatorname{opt}_{\theta}(\theta_t, \mathcal{S}),
\label{eq:eq1}
\end{align}
where $D(\cdot,\cdot)$ is a distance function, $T$ is the number of steps of the whole training trajectory and $\operatorname{opt}_{\theta}(\cdot)$ is the optimization operator for updating parameter ${\theta}$. Note that Eq.~\eqref{eq:eq1} is a bi-level problem where we need to learn the synthetic graphs $\mathcal{S}$ at the outer optimization and update model parameters $\theta_t$ at the inner optimization. To learn synthetic graphs that generalize to a distribution of model parameters $P_{\theta_{0}}$, we sample $\theta_0 \sim P_{\theta_{0}}$ and rewrite Eq.~\eqref{eq:eq1} as:
\begin{align} 
\min _{\mathcal{S}}  {\mathop{\mathbb{E}}_{\theta_{0}\sim P_{\theta_0}}} & \left[\sum_{t=0}^{T-1} D\left(\nabla_{\theta} \ell\left(f_{\theta_{t}}(\mathcal{S}), \mathcal{Y}'\right), \nabla_{\theta} \ell\left(f_{\theta_{t}}(\mathcal{T}), \mathcal{Y}\right)\right)\right], \nonumber\\ 
& \quad\quad\quad \text{s.t.}  \; \theta_{t+1} =   \operatorname{opt}_{\theta}(\theta_t, \mathcal{S}).
\label{eq:obj}
\end{align}
\vskip 0.5em
\noindent\textbf{Discussion.}  The aforementioned strategy has demonstrated promising performance on condensing image datasets~\cite{zhao2020dataset,zhao2021dataset}. However, it is not clear how to model the discrete graph structure. Moreover, the inherent bi-level optimization inevitably hinders its scalability. To tackle these shortcomings, we propose \method{} that models the structure as a probabilistic graph model and is optimized through one-step gradient matching. In the following subsections, we introduce the details of \method{}.

\subsection{Learning Discrete Graph Structure}\label{sec:discrete}
For graph classification, each graph in the dataset is composed of an adjacency matrix and a feature matrix. For simplicity, we use ${\bf X'}\in \mathbb{R}^{N'\times{n}\times{d}}$ to denote the node features in all synthetic graphs $\mathcal{S}$ and ${\bf A'}\in\{0,1\}^{N'\times{n}\times{n}}$ to indicate the graph structure information in $\mathcal{S}$. Note that $f_{\theta_t}$ can be instantiated as any graph neural network and it takes both graph structure and node features as input.
Then we rewrite the objective in Eq.~\eqref{eq:obj} as follows: 
\begin{align} 
\label{eq:eq3}
\min _{{\bf A}',{\bf X}'}  \mathop{\mathbb{E}}_{\theta_{0}\sim P_{\theta_0}}  & \left[\sum_{t=0}^{T-1} D\left(\nabla_{\theta} \ell\left(f_{\theta_{t}}({\bf A'}, {\bf X'}), \mathcal{Y}'\right), \nabla_{\theta} \ell\left(f_{\theta_{t}}(\mathcal{T}), \mathcal{Y}\right)\right)\right], \nonumber \\
& \quad\quad\quad \text{s.t.}  \; \theta_{t+1} =   \operatorname{opt}_{\theta}(\theta_t, \mathcal{S}),
\end{align}
where we aim to learn both graph structure ${\bf A}'$ and node features ${\bf X}'$. However, Eq.~\eqref{eq:eq3} is challenging to optimize as it requires a function that outputs binary values. To address this issue, we propose to model the graph structure as a probabilistic graph model with Bernoulli distribution. Note that in the following, we reshape ${\bf A}'$ from $N'\times{n}\times{n}$ to $N'\times{n^2}$ for the purpose of demonstration only. Specifically, for each entry ${\bf A}'_{ij} \in \{0,1\}$ in the adjacency matrix ${\bf A}'$, it follows a Bernoulli distribution:
\begin{equation}
    P_{{\bf \Omega}_{ij}}({\bf A}'_{ij}) = {\bf A}'_{ij} \sigma({\bf \Omega}_{ij}) + (1-{\bf A}'_{ij})\sigma(-{\bf \Omega}_{ij}),
\end{equation}
where $\sigma(\cdot)$ is the sigmoid function; ${\bf \Omega}_{ij} \in \mathbb{R}$ is the success probability of the Bernoulli distribution and also the parameter to be learned.
Since ${\bf A}'_{ij}$ is independent of all other entries,  the distribution of ${\bf A}'$ can be modeled as:
\begin{equation}
    P_{{\bf \Omega}}({\bf A}')=\prod_{i=1}^{N'}\prod_{j=1}^{n^2}  P_{{\bf \Omega}_{i j}}\left({\bf A}'_{i j}\right).
\end{equation}
Then, the objective in Eq.~\eqref{eq:obj} needs to be modified to
\begin{equation} 
\min _{{\bf A}', {\bf X}'}  \mathop{\mathbb{E}}_{\theta_{0}\sim P_{\theta_0}}\left[
\mathop{\mathbb{E}}_{{\bf A'}\sim P_{{\bf \Omega}}} \left[\ell({\bf A}'({\bf \Omega}), {\bf X}', \theta_0)\right]
\right].
\end{equation}
With the new parameterization, we obtain a function that outputs discrete values but it is not differentiable due to the involved sampling process. Thus, we employ the reparameterization method~\cite{maddison2016concrete}, binary concrete distribution, to refactor the discrete random variable into a differentiable function of its parameters and a random variable with fixed  distribution. Specifically, we first sample $\alpha \sim \operatorname{Uniform}(0,1)$, and  edge weight ${\bf A}'_{ij} \in [0,1]$ is calculated by:
\begin{equation} {\bf A}'_{ij}=\sigma\left(\left(\log \alpha-\log (1-\alpha)+{\bf \Omega}_{ij}\right) / \tau\right),
\end{equation}
where $\tau\in(0,\infty)$ is the temperature parameter that controls the continuous relaxation. As $\tau \rightarrow 0$, the random variable ${\bf A}'_{ij}$ smoothly approaches the Bernoulli distribution. In other words, we have $\lim _{\tau \rightarrow 0} P\left({\bf A}'_{i j}=1\right)=\sigma({\bf \Omega}_{ij})$. While small $\tau$ is necessary for obtaining discrete samples, large $\tau$ is useful in getting large gradients as suggested by~\cite{maddison2016concrete}.
In practice, we employ an annealing schedule~\cite{abid2019concrete} to gradually decrease the value of $\tau$ in training. With the reparameterization trick, the objective function becomes differentiable w.r.t. ${\bf \Omega}_{ij}$ with well-defined gradients. Then we rewrite our objective as:
\begin{small}
\begin{align} 
\label{eq:reparam}
& \min _{{\bf \Omega}, {\bf X}'}  \mathop{\mathbb{E}}_{\theta_{0}\sim P_{\theta_0}}\left[
\mathop{\mathbb{E}}_{\alpha\sim\operatorname{Uniform}(0,1)} \left[\ell({\bf A}'({\bf \Omega}), {\bf X}', \theta_0)\right]
\right] = \\
&  \mathop{\mathbb{E}}_{\theta_{0}}\left[
\mathop{\mathbb{E}}_{\alpha} \left[\sum_{t=0}^{T-1} D\left(\nabla_{\theta}  \ell\left(f_{\theta_{t}}({\bf A'}({\bf \Omega}), {\bf X'}), \mathcal{Y}'\right), \nabla_{\theta} \ell\left(f_{\theta_{t}}(\mathcal{T}), \mathcal{Y}\right)\right) \right]
\right], \nonumber \\
& \quad\quad\quad\quad\quad\quad\text{s.t.}  \; \theta_{t+1} =   \operatorname{opt}_{\theta}(\theta_t, \mathcal{S}). \nonumber
\end{align}
\end{small}


\subsection{One-Step Gradient Matching}\label{sec:one_step}
The vanilla gradient matching scheme in Eq.~\eqref{eq:obj} presents a bi-level optimization problem. To solve this problem,  we need to  update the synthetic graphs $\mathcal{S}$ at the outer loop and then optimize the network parameters $\theta_t$ at the inner loop. The nested loops heavily impede the scalability of the condensation method, which motivates us to design a new strategy for efficient condensation. In this work, we propose a \textit{one-step gradient matching} scheme where we only match the network gradients for the model initializations $\theta_0$ while discarding the training trajectory of $\theta_t$. Essentially, this strategy approximates the overall gradient matching loss for $\theta_t$ with the initial matching loss at the first epoch, which we term as \textit{one-step matching loss}. The intuition is: the one-step matching loss informs us about the direction to update the synthetic data, in which, we have empirically observed a strong decrease in the cross-entropy loss (on real samples) obtained from the model trained on synthetic data. Hence, we can drop the summation symbol $\sum_{t=0}^{T-1}$ in Eq.~\eqref{eq:reparam} and simplify Eq.~\eqref{eq:reparam} as follows:
\begin{small}
\begin{align}
\min _{{\bf \Omega}, {\bf X}'} \mathop{\mathbb{E}}_{\theta_{0}}\left[
\mathop{\mathbb{E}}_{\alpha} \left[ D\left(\nabla_{\theta}  \ell\left(f_{\theta_{0}}({\bf A'}({\bf \Omega}), {\bf X'}), \mathcal{Y}'\right), \nabla_{\theta} \ell\left(f_{\theta_{0}}(\mathcal{T}), \mathcal{Y}\right)\right) \right]\right], 
\end{align}
\end{small}where we sample $\theta_0\sim P_{\theta_0}$ and $\alpha\sim\operatorname{Uniform}(0,1)$. 
Compared with Eq.~\eqref{eq:reparam}, one-step gradient matching avoids the expensive nested-loop optimization and directly updates the synthetic graph $\mathcal{S}$. It greatly simplifies the condensation process. In practice, as shown in Section~\ref{sec:ablation}, we find this strategy yields comparable performance to its bi-level counterpart while enabling much more efficient condensation. Next, we provide theoretical analysis to understand the rationality of the proposed one-step gradient matching scheme.





\vskip 0.5em
\noindent\textbf{Theoretical Understanding.} We denote the cross entropy loss on the real graphs as $\ell_\mathcal{T}(\theta)=\sum_{i}\ell_i({\bf A}_{(i)}$, ${\bf X}_{(i)}, \theta)$ and that on synthetic graphs as $\ell_{\mathcal{S}}(\theta)=\ell_{\mathcal{S}}({\bf A}'_{(i)}, {\bf X}'_{(i)}, \theta)$.
Let $\theta^*$ denote the optimal parameter and $\theta_t$ be the parameter trained on $\mathcal{S}$ at the $t$-th epoch by optimizing $\ell_{\mathcal{S}}(\theta)$. 
For notation simplicity, we assume that ${\bf A}$ and ${\bf A}'$ are already normalized. The matrix norm $\|\cdot\|$ is the Frobenius norm. We focus on the GNN of Simple Graph Convolutions  (SGC)~\cite{wu2019simplifying} to study our problem since SGC has a simpler architecture but shares a similar filtering pattern as GCN. 

\begin{theorem}
\label{thm:graph}
When we use a $K$-layer SGC as the GNN used in condensation, i.e., $f_\theta({\bf A}_{(i)}, {\bf X}_{(i)})= \text{Pool}({\bf A}_{(i)}^K{\bf X}_{(i)}{\bf W}_1){\bf W}_2$ with $\theta=[{\bf W}_1;{\bf W}_2]$ and assume that all network parameters satisfy $\|\theta\|^2\leq M^2 (M>0)$, we have
\begin{small}
\begin{align}
\min _{t=0,1, \ldots, T-1}  & \ell_\mathcal{T}\left(\theta_{t}\right)- 
\ell_\mathcal{T}\left(\theta^{*}\right) \leq   \sum_{t=0}^{T-1} \frac{\sqrt{2}M}{T} \|\nabla_{\theta} \ell_\mathcal{T}\left(\theta_{t}\right)- \nabla_{\theta} \ell_{\mathcal{S}}\left(\theta_{t}\right)\|  \nonumber \\
& +  \frac{3M}{2\sqrt{T}}\cdot\frac{C-1}{C N'}\sqrt{\sum_{i}\gamma_i\|{{\bf 1}^\top{\bf A}'^K_{(i)}}{\bf X}'_{(i)}\|^2} 
\end{align}
\end{small}where $\gamma_{i}=1$ if we use sum pooling in $f_\theta$;  $\gamma_{i}=\frac{1}{n_i}$ if we use mean pooling, with $n_i$ as the number of nodes in the $i$-th synthetic graph. 
\end{theorem}

We provide the proof of Theorem 1 in Appendix~\ref{appendix:thm1}. Theorem 1 suggests that the smallest gap between the resulted loss (by training on synthetic graphs) and the  optimal loss has an upper bound. This upper bound depends on two terms: (1) the difference of gradients w.r.t. real data and synthetic data and (2) the norm of input matrices. Thus, the theorem justifies that reducing the gradient difference w.r.t real and synthetic graphs can help learn desirable synthetic data that preserves sufficient information to train GNNs well. Based on Theorem 1, we have the following proposition.
\setcounter{theorem}{0}
\begin{proposition}
\label{prop:1}
Assume the largest gradient gap happens at $0$-th epoch, i.e., $\|\nabla_{\theta} \ell_\mathcal{T}\left(\theta_{0}\right)- \nabla_{\theta} \ell_S\left(\theta_{0}\right)\| = \max\limits_{t}\|\nabla_{\theta} \ell_\mathcal{T}\left(\theta_{t}\right)- \nabla_{\theta} \ell_S\left(\theta_{t}\right)\|$ with $t=0,1,\ldots,T-1$, we have
\begin{small}
\begin{align}
\min _{t=0,1,\ldots,T-1}  & \ell_\mathcal{T}\left(\theta_{t}\right)-
\ell_\mathcal{T}\left(\theta^{*}\right) \leq   {\sqrt{2}M}\|\nabla_{\theta} \ell_\mathcal{T}\left(\theta_{0}\right)- \nabla_{\theta} \ell_S\left(\theta_{0}\right)\|    \nonumber \\
& +  \frac{3M}{2\sqrt{T}}\cdot\frac{C-1}{C N'}\sqrt{\sum_{i}\gamma_i\|{{\bf 1}^\top{\bf A}'^K_{(i)}}{\bf X}'_{(i)}\|^2}.
\label{eq:prop}
\end{align}
\end{small}
\end{proposition}
We omit the proof for the proposition since it is straightforward. The above proposition suggests that the smallest gap between the $\ell_\mathcal{T}(\theta_t)$ and $\ell_\mathcal{T}(\theta^*)$ is bounded by the one-step matching loss and the norm $\|{{\bf 1}^\top{\bf A}'^K_{(i)}}{\bf X}'_{(i)}\|^2$. As we will show in Section~\ref{sec:terms}, when using mean pooling, the second term tend to have a smaller scale than the first one and can be neglected; the second term matters more when we use sum pooling. Hence, we solely optimize the one-step gradient matching loss for GNNs with mean pooling and additionally include the second term (the norm of input matrices) as a regularization for GNNs with sum pooling. As such, when we consider the optimal loss $\ell_\mathcal{T}(\theta^*)$ as a constant, reducing the one-step matching loss indeed learns synthetic graphs that lead to a smaller loss on real graphs.
This demonstrates the rationality of one-step gradient matching from a theoretical perspective.


\setcounter{theorem}{0}
\begin{remark}
Note that the spectral analysis from~\cite{wu2019simplifying} demonstrated that both GCN and SGC share similar graph filtering behaviors. Thus practically, we extend the one-step gradient matching loss from $K$-layer SGC to $K$-layer GCN and observe that the proposed framework works well under the non-linear scenario.
\end{remark}
\begin{remark}
While we focus on the graph classification task, it is straightforward to extend our framework to node classification and we obtain similar conclusions for node classification as shown in Theorem 2 in Appendix~\ref{app:node}.
\end{remark}


\subsection{Final Objective and Training Algorithm}
In this subsection, we describe the final objective function and the detailed training algorithm. We note that
the objective in Eq.~\eqref{eq:reparam} involves two nested expectations, we adopt Monte Carlo to approximately optimize the objective function. Together with one-step gradient matching, we have 
\begin{align} 
& \min _{{\bf {\bf \Omega}}, {\bf X}'}  \mathop{\mathbb{E}}_{\theta_{0}\sim P_{\theta_0}}\mathop{\mathbb{E}}_{\alpha\sim\operatorname{Uniform}(0,1)} \left[\left[  
\ell({\bf A}'({\bf {\bf \Omega}}), {\bf X}', \theta_0)\right]
\right]\\
& \approx \sum_{k_1=1}^{K_1}\sum_{k_2=1}^{K_2}  D\left(\nabla_{\theta}  \ell\left(f_{\theta_{0}}({\bf A'}({\bf {\bf \Omega}}), {\bf X'}), \mathcal{Y}'\right), \nabla_{\theta} \ell\left(f_{\theta_{0}}(\mathcal{T}), \mathcal{Y}\right)\right) \nonumber
\end{align}
where $K_1$ is the number of sampled  model initializations and $K_2$ is the number of sampled graphs. We find that $K_2=1$ is able to yield good performance in our experiments.

\vskip 0.5em
\noindent\textbf{Regularization.} In addition to the one-step gradient matching loss, we note that the proposed \method{} can be easily integrated with various priors as regularization terms. In this work, we focus on exerting sparsity regularization on the adjacency matrix, since a denser adjacency matrix will lead to higher cost for training graph neural networks. Specifically, we penalize the difference of the sparsity between $\sigma({\Omega})$ and a given sparsity $\epsilon$:
\begin{equation}
    \ell_\text{reg} = \operatorname{max}(\frac{1}{|{\bf \Omega}|}\sum_{i,j}\sigma({\bf \Omega}_{ij}) - \epsilon, 0).
\end{equation}
We initialize $\sigma({\bf\Omega})$ and ${\bf X}'$ as randomly sampled training graphs\footnote{If an entry in the real adjacency matrix is 1, the corresponding value in ${\bf\Omega}$ is initialized as a large value, e.g.,5.} and set $\epsilon$ to the average sparsity of initialized $\sigma({\bf\Omega})$ so as to maintain a low sparsity.  On top of that, as we discussed earlier in Section~\ref{sec:one_step}, we include the following regularization for GNNs with sum pooling:
\begin{equation}
\ell_\text{reg2} = \frac{3}{2\sqrt{2T}}\cdot\frac{C-1}{C N'}\sqrt{\sum_{i}\|{{\bf 1}^\top{\bf A}'^K_{(i)}}{\bf X}'_{(i)}\|^2}
\label{eq:reg_norm}
\end{equation}

\begin{table*}[t]
\caption{The classification performance comparison to baselines. We report the ROC-AUC for the first three datasets and accuracies (\%) for others. \textit{Whole Dataset} indicates the performance with original dataset.}
\vskip -1em
\label{tab:main}
\small
\begin{tabular}{@{}lcccccc|c|c@{}}
\toprule
                                                                                  & Graphs/Cls. & Ratio  & Random       & Herding               & K-Center     & DCG                  & \method{}                  & Whole Dataset                 \\ \midrule
\multirow{3}{*}{\begin{tabular}[c]{@{}l@{}}ogbg-molbace\\ (ROC-AUC)\end{tabular}} & 1          & 0.2\%  & 0.580$\pm$0.067 & 0.548$\pm$0.034          & 0.548$\pm$0.034 & 0.623$\pm$0.046         & \textbf{0.657$\pm$0.034} & \multirow{3}{*}{0.714$\pm$0.005} \\
                                                                                  & 10         & 1.7\%  & 0.598$\pm$0.073 & 0.639$\pm$0.039          & 0.591$\pm$0.056 & 0.655$\pm$0.033         & \textbf{0.674$\pm$0.035} &                               \\ 
                                                                                  & 50         & 8.3\%  & 0.632$\pm$0.047 & 0.683$\pm$0.022          & 0.589$\pm$0.025 & 0.652$\pm$0.013         & \textbf{0.688$\pm$0.012} &                               \\ \midrule
\multirow{3}{*}{\begin{tabular}[c]{@{}l@{}}ogbg-molbbbp\\ (ROC-AUC)\end{tabular}} & 1          & 0.1\%  & 0.519$\pm$0.016 & 0.546$\pm$0.019          & 0.546$\pm$0.019 & 0.559$\pm$0.044         & \textbf{0.581$\pm$0.005} & \multirow{3}{*}{0.646$\pm$0.004} \\
                                                                                  & 10         & 1.2\%  & 0.586$\pm$0.040 & \textbf{0.605$\pm$0.019} & 0.530$\pm$0.039 & 0.568$\pm$0.032         & \textbf{0.605$\pm$0.008} &                               \\
                                                                                  & 50         & 6.1\%  & 0.606$\pm$0.020 & 0.617$\pm$0.003          & 0.576$\pm$0.019 & 0.579$\pm$0.032         & \textbf{0.620$\pm$0.007} &                               \\ \midrule
\multirow{3}{*}{\begin{tabular}[c]{@{}l@{}}ogbg-molhiv\\ (ROC-AUC)\end{tabular}}  & 1          & 0.01\%  & 0.719$\pm$0.009 & 0.721$\pm$0.002          & 0.721$\pm$0.002 & 0.718$\pm$0.013         & \textbf{0.726$\pm$0.003} & \multirow{3}{*}{0.757$\pm$0.007} \\
                                                                                  & 10         & 0.06\%  & 0.720$\pm$0.011 & 0.725$\pm$0.006          & 0.713$\pm$0.009 & 0.728$\pm$0.002         & \textbf{0.728$\pm$0.005} &                               \\
                                                                                  & 50         & 0.3\%  & 0.721$\pm$0.014 & 0.725$\pm$0.003          & 0.725$\pm$0.006 & 0.726$\pm$0.010         & \textbf{0.731$\pm$0.004} &                               \\ \midrule
\multirow{3}{*}{\begin{tabular}[c]{@{}l@{}}DD\\ (Accuracy)\end{tabular}}          & 1          & 0.2\%  & 57.69$\pm$4.92  & 61.97$\pm$1.32           & 61.97$\pm$1.32  & 58.81$\pm$2.90          & \textbf{70.42$\pm$2.21}  & \multirow{3}{*}{78.92$\pm$0.64}    \\
                                                                                  & 10         & 2.1\%  & 64.69$\pm$2.55  & 69.79$\pm$2.30           & 63.46$\pm$2.38  & 61.84$\pm$1.44          & \textbf{73.53$\pm$1.13}  &                               \\
                                                                                  & 50         & 10.6\% & 67.29$\pm$1.53  & 73.95$\pm$1.70           & 67.41$\pm$0.92  & 61.27$\pm$1.01          & \textbf{77.04$\pm$1.86}  &                               \\ \midrule
\multirow{3}{*}{\begin{tabular}[c]{@{}l@{}}MUTAG\\ (Accuracy)\end{tabular}}       & 1          & 1.3\%  & 67.47$\pm$9.74 & 70.84$\pm$7.71           & 70.84$\pm$7.71  & 75.00$\pm$8.16          & \textbf{82.21$\pm$1.61}  & \multirow{3}{*}{88.63$\pm$1.44}    \\
                                                                                  & 10         & 13.3\% & 77.89$\pm$7.55  & 80.42$\pm$1.89           & 81.00$\pm$2.51  & {82.66$\pm$0.68} & \textbf{82.76$\pm$2.31}           &                               \\
                                                                                  & 20         & 26.7\% & 78.21$\pm$5.13  & 80.00$\pm$1.10           & 82.97$\pm$4.91  & 82.89$\pm$1.03          & \textbf{83.26$\pm$2.34}  &                               \\ \midrule
\multirow{3}{*}{\begin{tabular}[c]{@{}l@{}}NCI1\\ (Accuracy)\end{tabular}}        & 1          & 0.1\%  & 51.27$\pm$1.22  & 53.98$\pm$0.67           & 53.98$\pm$0.67  & 51.14$\pm$1.08          & \textbf{56.58$\pm$0.48}  & \multirow{3}{*}{71.70$\pm$0.20}    \\
                                                                                  & 10         & 0.6\%  & 54.33$\pm$3.14  & 57.11$\pm$0.56           & 53.21$\pm$1.44  & 51.86$\pm$0.81          & \textbf{58.02$\pm$1.05}  &                               \\
                                                                                  & 50         & 3.0\%  & 58.51$\pm$1.73  & 58.94$\pm$0.83           & 56.58$\pm$3.08  & 52.17$\pm$1.90          & \textbf{60.07$\pm$1.58}  &                               \\ \midrule
\multirow{3}{*}{\begin{tabular}[c]{@{}l@{}}CIFAR10\\ (Accuracy)\end{tabular}}     & 1          & 0.06\%  & 15.61$\pm$0.52  & 22.38$\pm$0.49           & 22.37$\pm$0.50  & 21.60$\pm$0.42          & \textbf{24.70$\pm$0.70}  & \multirow{3}{*}{50.75$\pm$0.14}    \\
                                                                                  & 10         & 0.2\%  & 23.07$\pm$0.76  & 28.81$\pm$0.35           & 20.93$\pm$0.62  & 29.27$\pm$0.77          & \textbf{30.70$\pm$0.23}  &                               \\
                                                                                  & 50         & 1.1\%  & 30.56$\pm$0.81  & 33.94$\pm$0.37           & 24.17$\pm$0.51  & 34.47$\pm$0.52          & \textbf{35.34$\pm$0.14}  &                               \\  \midrule
\multirow{3}{*}{\begin{tabular}[c]{@{}l@{}}E-commerce\\ (Accuracy)\end{tabular}}      & 1  & 0.2\% & 51.31$\pm$2.89 & 52.18$\pm$0.25 & 52.36$\pm$0.38 & 57.14$\pm$1.72 & \textbf{60.82$\pm$1.23}   & \multirow{3}{*}{69.25$\pm$0.50}    \\
 & 10 & 0.9\% & 54.99$\pm$2.74 & 56.83$\pm$0.87 & 56.49$\pm$0.36 & 61.03$\pm$1.32 & \textbf{64.73$\pm$1.34} \\
 & 20 & 3.6\% & 57.80$\pm$3.58 & 62.56$\pm$0.71 & 62.76$\pm$0.45 & 64.92$\pm$1.35 & \textbf{67.71$\pm$1.22} &                               \\ \bottomrule
\end{tabular}
\vskip -1em
\end{table*}

\vskip 0.5em
\noindent\textbf{Training Algorithm.} We provide the details of our proposed framework in Algorithm~\ref{alg} in Appendix~\ref{app:alg}. Specifically, we sample $K_1$ model initializations $\theta_0$ to perform one-step gradient matching. Following the convention in DC~\cite{zhao2020dataset}, we match gradients and update synthetic graphs for each class separately in order to make matching easier. For class $c$, we first retrieve the synthetic graphs of that class, denoted as $({\bf A}_{c}',{\bf X}_{c}', \mathcal{Y}_{c}')\sim \mathcal{S}$, and sample a batch of real graphs $({\bf A}_c,{\bf X}_c, \mathcal{Y}_c)$. We then forward them to the graph neural network and calculate the one-step gradient matching loss together with the regularization term. Afterwards, ${\bf \Omega}$ and ${\bf X}'$ are updated via gradient descent. It is worth noting that the training process for each class can be run in parallel since the graph updates for one class is independent of another class.


\vskip 0.5em
\noindent\textbf{Comparison with DC.} Recall that the gradient matching scheme in DC involves a complex bi-level optimization. If we denote the number of inner-iterations as $\tau_i$ and that of outer-iterations as $\tau_o$,  its computational complexity can be $\tau_i\times \tau_o$ of our method. Thus DC is significantly slower than \method{}. In addition to speeding up condensation,  \method{} removes the burden of tuning some hyper-parameters, i.e., the number of iterations for outer/inner optimization and learning rate for updating $f_\theta$, which can potentially save us enormous training time when learning larger synthetic sets.

\vskip 0.5em
\noindent\textbf{Comparison with Coreset Methods.} Coreset methods~\cite{welling2009herding,sener2017active} select representative data samples based on some heuristics calculated on the pre-trained embedding. Thus, it requires training the model first. Given the cheap cost on calculating and ranking heuristics, the major computational bottleneck for coreset method is on pre-training the neural network for a certain number of iterations.  
Likewise, our proposed \method{} has comparable complexity because it also needs to forward and backward the neural network for multiple iterations.  
Thus, their efficiency difference majorly depends on how many epochs we run for learning synthetic graphs in \method{} and for pre-training the model embedding in coreset methods.
In practice, we find that \method{} even requires less training cost than the coreset methods as shown in Section~\ref{sec:time}. 

\section{Experiment}\label{sec:experiments}
In this section, we conduct experiments to evaluate \method{}. Particularly, we aim to answer the following questions:
(a) how well can we condense a graph dataset and (b) how efficient is \method{}. Our code can be found in the supplementary files.

\subsection{Experimental settings}

\textbf{Datasets.} 
 To evaluate the performance of our method, we use multiple molecular datasets from Open Graph Benchmark (OGB)~\cite{hu2020open} and TU Datasets (DD, MUTAG and NCI1)~\cite{morris2020tudataset} for graph-level property classification, and one superpixel dataset CIFAR10~\cite{dwivedi2020benchmarkgnns}. 
 We also introduce a real-world e-commerce dataset. In particular, we randomly sample 1,109 sub-graphs from a large, anonymized internal knowledge graph. Each sub-graph is created from the ego network of a random selected product on the e-commerce website. We form a binary classification problem aiming at predicting the product category of the central product node in each sub-graph. 
 We use the public splits for OGB datasets and CIFAR10. For TU Datasets and the e-commerce dataset, we randomly split the graphs into 80\%/10\%/10\% for training/validation/test. Detailed dataset statistics are shown in Appendix~\ref{appendix:data}. 

\vskip 0.5em
\noindent\textbf{Baselines.} 
We compare our proposed methods with four baselines that produce discrete structures: three coreset methods (\textit{Random}, \textit{Herding}~\citep{welling2009herding} and \textit{K-Center}~\citep{farahani2009facility,sener2017active}), and a \textit{dataset condensation} method DCG~\cite{zhao2020dataset}: (a) {Random:} it randomly picks graphs from the training dataset. (b) {Herding:} it selects samples that are closest to the cluster center. Herding is often used in replay-based methods for continual learning ~\citep{rebuffi2017icarl,castro2018end}.
(c) {K-Center}: it selects the center samples to minimize the largest distance between a sample and its nearest center. 
(d) {DCG}: As vanilla DC~\cite{zhao2020dataset} cannot generate discrete structure, we randomly select graphs from training and apply DC to learn the features for them, which we term as DCG. We use the implementations provided by~\citet{zhao2020dataset} for Herding, K-Center and DCG. Note that coreset methods only select existing samples from training while DCG learns the node features.

\vskip 0.5em
\noindent\textbf{Evaluation Protocol.} 
To evaluate the effectiveness of the proposed method, we test the classification performance of GNNs trained with condensed graphs on the aforementioned graph datasets. Concretely, it involves three stages: (1) learning synthetic graphs, (2) training a GCN on the synthetic graphs and (3) test the performance of GCN. We first generate the condensed graphs following the procedure in Algorithm 1. Then we train a GCN classifier with the condensed graphs. Finally we evaluate its classification performance on the real graphs from test set. For baseline methods, we first get the selected/condensed graphs and then follow the same procedure.  We repeat the generation process of condensed graphs 5 times with  different random seeds and train GCN on these graphs with 10 different random seeds. In all experiments, we report the mean and standard deviation of these results.

\vskip 0.5em
\noindent\textbf{Parameter Settings.}
When learning the synthetic graphs,  we adopt 3-layer GCN with 128 hidden units as the model for gradient matching. The learning rates for structure and feature parameters are set to 1.0 (0.01 for ogbg-molbace and CIFAR10) and 0.01, respectively. We set $K_1$ to 1000 and $\beta$ to 0.1. Additionally, we use mean pooling to obtain graph representation for all datasets except ogbg-molhiv. We use sum pooling for ogbg-molhiv as it achieves better classification performance on the real dataset. During the test stage, we use GCN with the same architecture and we train the model for 500 epochs (100 epochs for ogbg-molhiv) with an initial learning rate of 0.001. 

\vskip -1em
\subsection{Performance with Condensed Graphs}

\subsubsection{Classification Performance Comparison.} To validate the effectiveness of the proposed framework, we measure the classification performance of GCN trained on condensed graphs. Specifically, we vary the number of learned synthetic graphs per class in the range of $\{1, 10, 50\}$ ($\{1, 10, 20\}$ for MUTAG and E-commerce) and train a GCN on these graphs. Then we evaluate the classification performance of the trained GCN on the original test graphs. Following the convention in OGB~\cite{hu2020open}, we report the ROC-AUC metric for ogbg-molbace, ogbg-molbbbp and ogbg-molhiv; for other datasets we report the classification accuracy (\%).  The results  are summarized in Table~\ref{tab:main}. Note that the \textit{Ratio} column presents the ratio of synthetic graphs to original graphs and we name it as \textit{condensation ratio}; the \textit{Whole Dataset} column shows the GCN performance achieved by training on the original dataset. From the table, we make the following observations:
\begin{compactenum}[(a)]
    \item The proposed \method{} consistently achieves better performance than the baseline methods under different condensation ratios and different datasets. Notably, when generating only 2 graphs on ogbg-molbace dataset (0.2\%), we achieve an ROC-AUC of 0.657 while the performance on full training set is 0.714, which means we approximate 92\% of the original performance with only 0.2\% data. Likewise, we are able to approximate 96.5\% of the original performance on ogbg-molhiv with 0.3\% data.  By contrast,  baselines underperform our method by a large margin. Similar observations can be made on other datasets, which demonstrates the effectiveness of learned synthetic graphs in preserving the information of the original dataset.
    \item Increasing the number of synthetic graphs can improve the classification performance. For example, we can approximate the original performance by 89\%/93\%/98\% with 0.2\%/2.1\%/10.6\% data on DD. More synthetic samples indicate more learnable parameters that can preserve the information residing in the original dataset and present more diverse patterns that can help train GNNs better. This observation is in line with our experimental results in Section~\ref{sec:increase}. 
    \item The performance on CIFAR10 is less promising due to the limit number of synthetic graphs. We posit that the dataset has more complex topology and feature information and thus requires more parameters to preserve sufficient information. However, we note that our method still outperforms the baseline methods especially when producing only 1 sample per class, which suggests that our method is much more data-efficient. Moreoever, we are able to promote the performance on CIFAR10 by learning a larger synthetic set as shown in  Section~\ref{sec:increase}.
    \item Learning both synthetic graph structure and node features is necessary for preserving the information in original graph datasets. By checking the performance DCG, which only learns node features based on randomly selected graph structure, we see that DCG underperforms \method{} by a large margin in most cases. This indicates that learning node features solely is sub-optimal for condensing graphs.  
\end{compactenum}


\subsubsection{Efficiency Comparison}\label{sec:time}
Since one of our goals is to enable scalable dataset condensation, we now evaluate the efficiency of \method{}. We compare \method{} with the coreset method Herding, as it is less time-consuming than DCG and generally achieves better performance than other baselines. We adopt the same setting as in Table~\ref{tab:main}: 1000 iterations for \method{}, i.e., $K_1=1000$, and  500 epochs (100 epochs for ogbg-molhiv) for pre-training the graph convolutional network as required by Herding. We also note that pre-training the neural network need to go over the whole dataset at every epoch while \method{} only processes a batch of graphs. In Table~\ref{tab:time}, we report the running time on an NVIDIA V100 GPU for CIFAR10, ogbg-molhiv and DD. From the table, we 
make two observations:
\begin{compactenum}[(a)]
    \item  \method{} can be faster than Herding. In fact, \method{} requires less training time in all the cases except in DD with 50 graphs per class. Herding needs to fully training the model on the whole dataset to obtain good-quality embedding, which can be quite time-consuming. On the contrary, \method only requires matching gradients for $K_1$ initializations and does not need to fully train the model on the large real dataset.
    \item The running time of \method{} increases with the increase of the number of synthetic graphs $N'$. It is because \method{}   processes the condensed graphs at each iteration, of which the time complexity is $O(N'L(n^2d+nd^2))$ for an $L$-layer GCN. Thus, the additional complexity depends on $N'$. By contrast, the increase of $N'$  has little impact on Herding since the process of selecting samples based on pre-defined heuristic is very fast. 
    \item The average nodes in synthetic graph $n$ also impacts the training cost of \method{}. For instance, the training cost on ogbg-molhiv ($n$=26) is much lower than that on DD ($n$=285), and the gap of cost between the two methods on ogbg-molhiv and DD is very different. As mentioned earlier, it is because the complexity of the forward process in GCN is $O(N'L(n^2d+nd^2))$ for $N'$ condensed graphs with node size of $n$. 
\end{compactenum}
To summarize, the efficiency difference of Herding and \method{} depends on the number of condensed/selected samples and the training iterations adopted in practice and we empirically found that \method{} consumes less training cost.

\begin{table}[t]
\small
\caption{Comparison of running time (minutes). }
\vskip -1em
\label{tab:time}
\begin{tabular}{@{}c|cc|cc|cc@{}}
\toprule
      & \multicolumn{2}{c}{CIFAR10} & \multicolumn{2}{|c}{ogbg-molhiv}   & \multicolumn{2}{|c}{DD}    \\ \midrule
G./Cls. & Herding        & \method{}        & Herding        & \method{}   & Herding &  \method{}    \\ 
1      & 44.5m            & 4.7m          & 4.3m            & 0.66m  &1.6m & 1.5m    \\
10     & 44.5m            & 4.9m       & 4.3m           & 0.67m       & 1.6m & 1.5m \\
50     & 44.5m            & 5.7m       & 4.3m           & 0.68m   & 1.6m & 2.0m \\ \bottomrule
\end{tabular}
\vskip -1em
\end{table}



\begin{figure*}[!t]%
\centering
  \subfloat[Learning larger synthetic set.]{\label{fig:increase}{\includegraphics[width=0.26\linewidth]{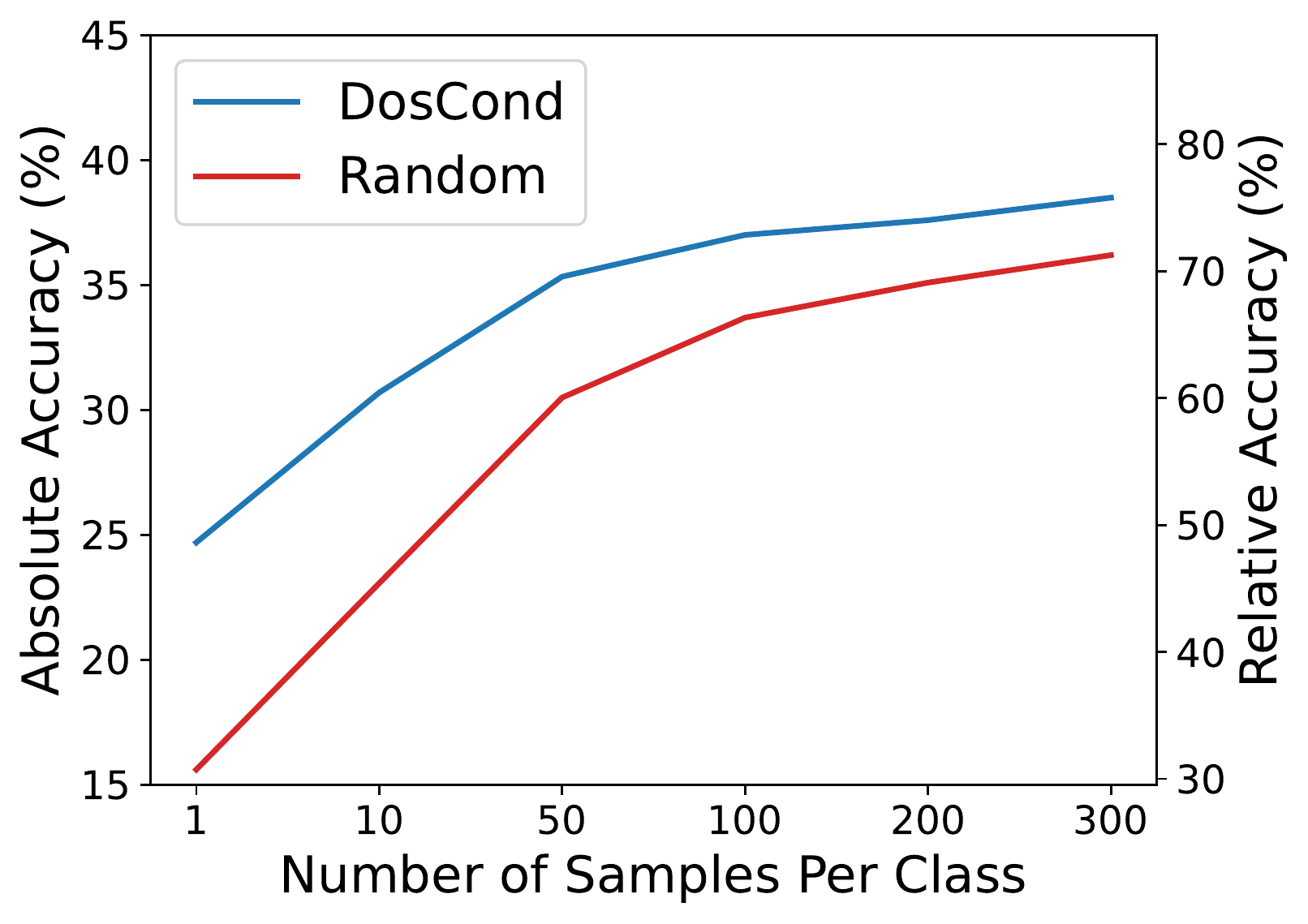} }}%
  \subfloat[One-Step v.s. bi-Level matching]{\label{fig:one_step}{\includegraphics[width=0.24\linewidth]{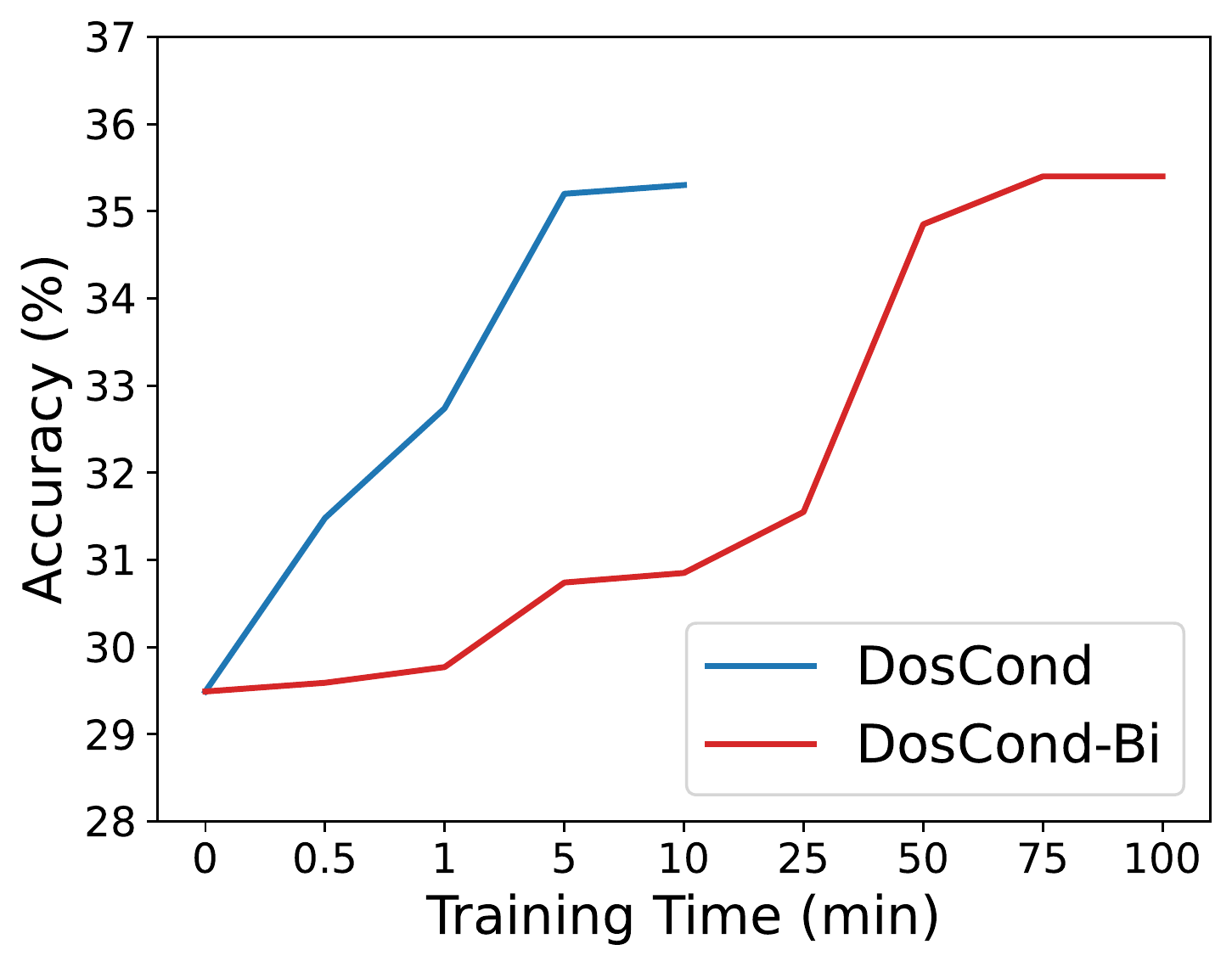} }}%
 \subfloat[Varying $\beta$ on DD]{\label{fig:dd}{\includegraphics[width=0.24\linewidth]{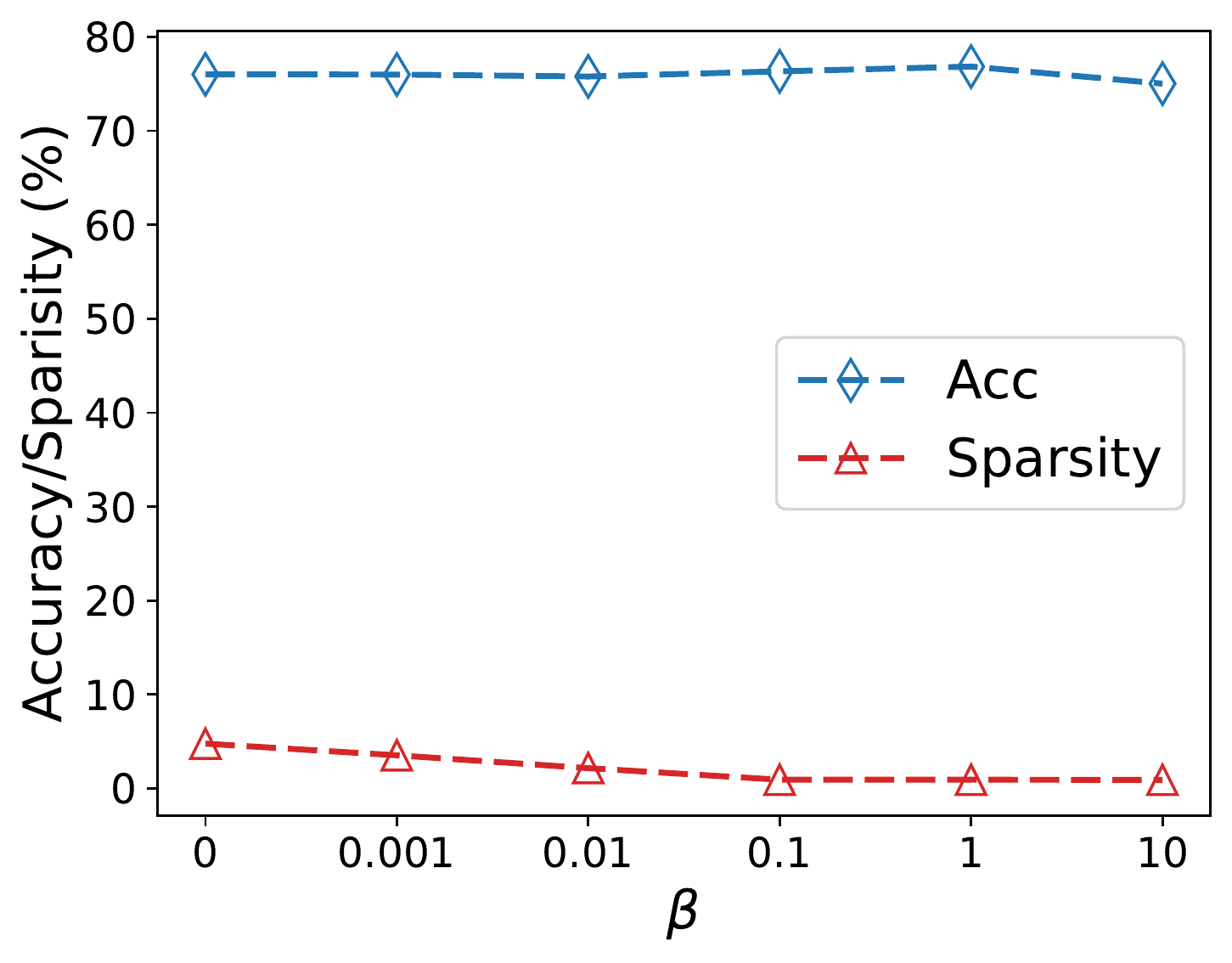} }}%
 \subfloat[Varying $\beta$ on NCI1]{\label{fig:nci1}{\includegraphics[width=0.24\linewidth]{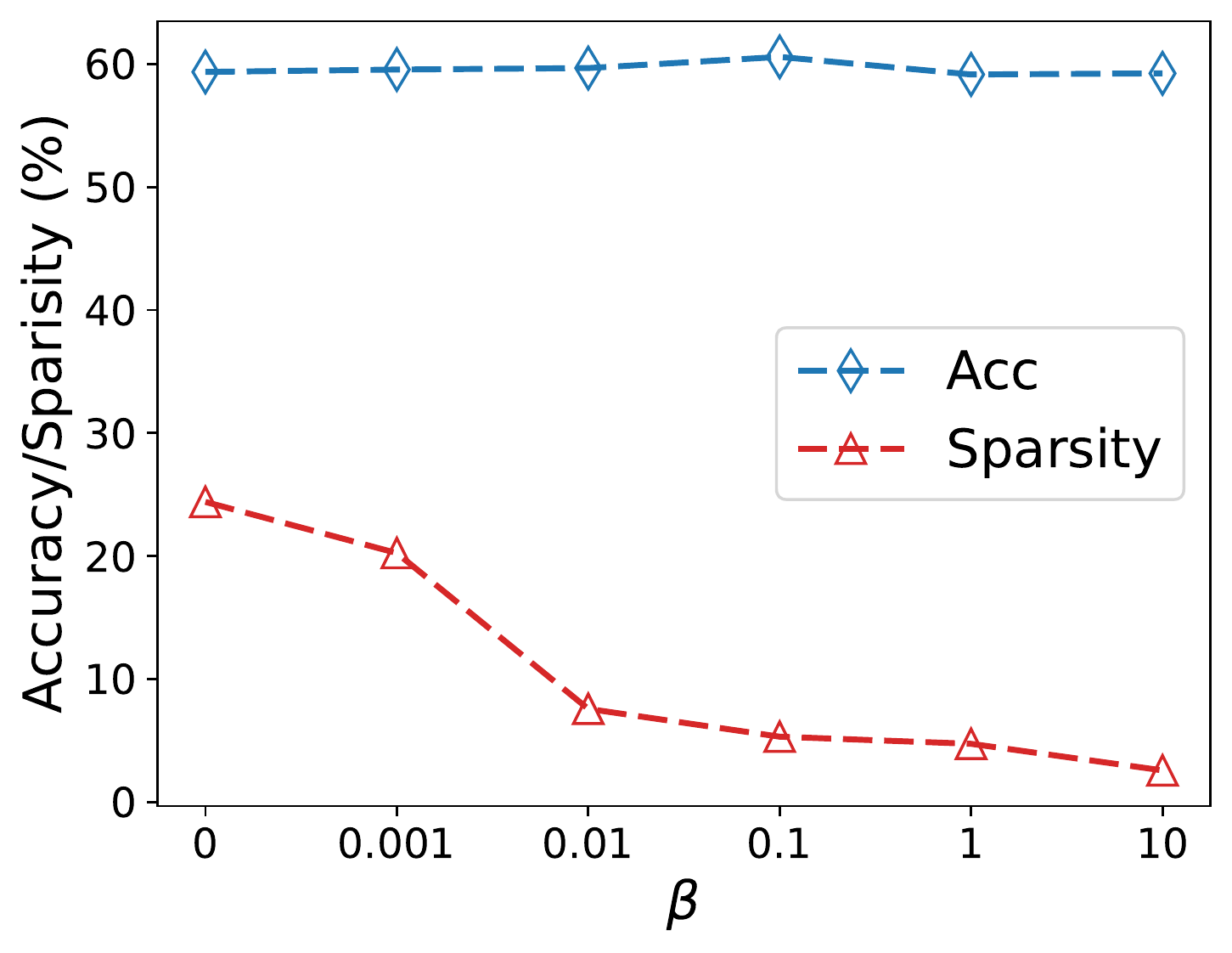} }}%
 \qquad
\vskip -1.2em
\caption{Parameter analysis w.r.t. the sparsity regularization.}  
\label{fig:param}
\vskip -1.5em
\end{figure*}

\begin{figure*}[t]%
\centering
 \subfloat[Random]{\label{fig:one_step}{\includegraphics[width=0.21\linewidth]{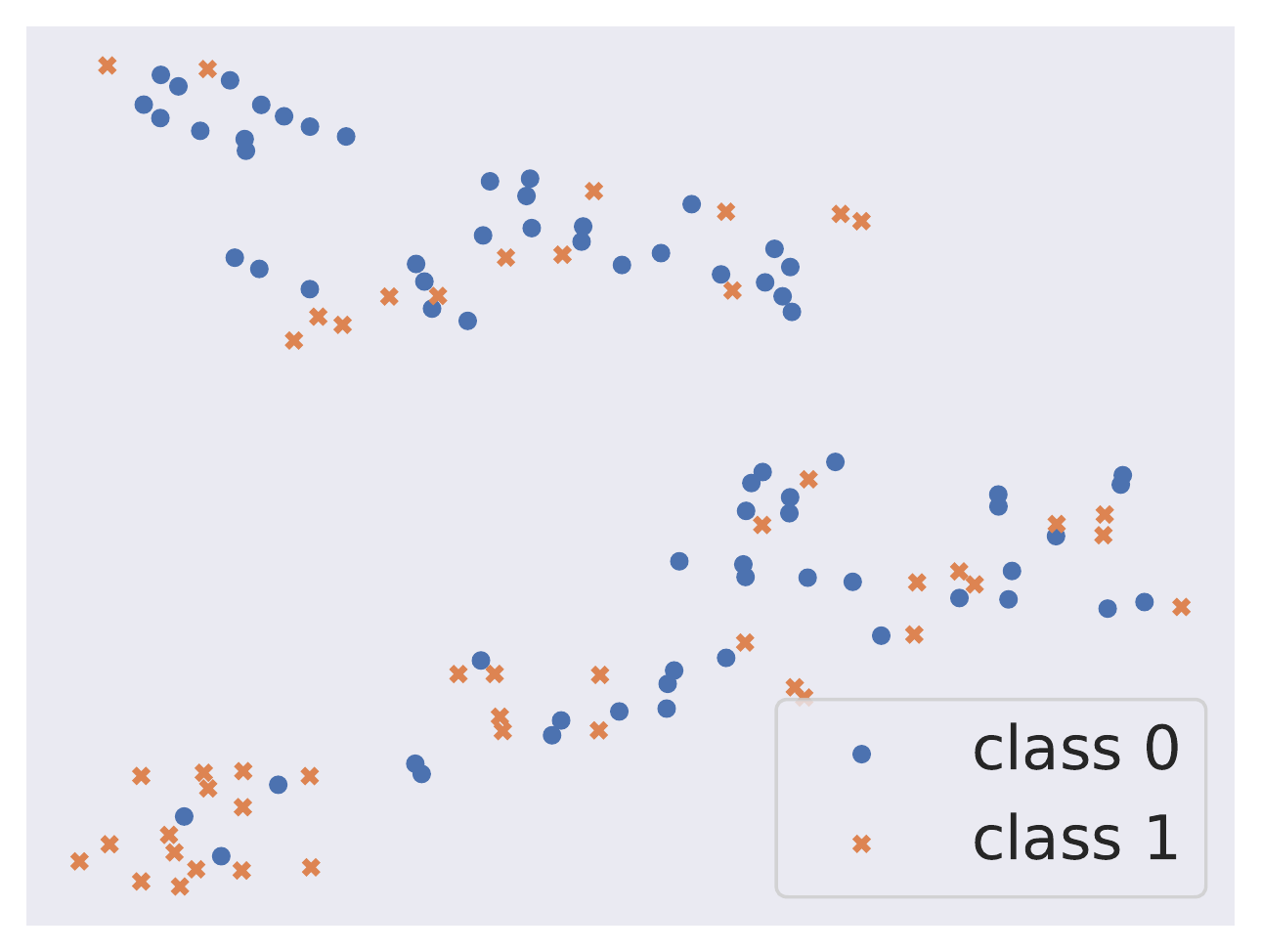} }}%
 \subfloat[DCG]{\label{fig:one_step}{\includegraphics[width=0.21\linewidth]{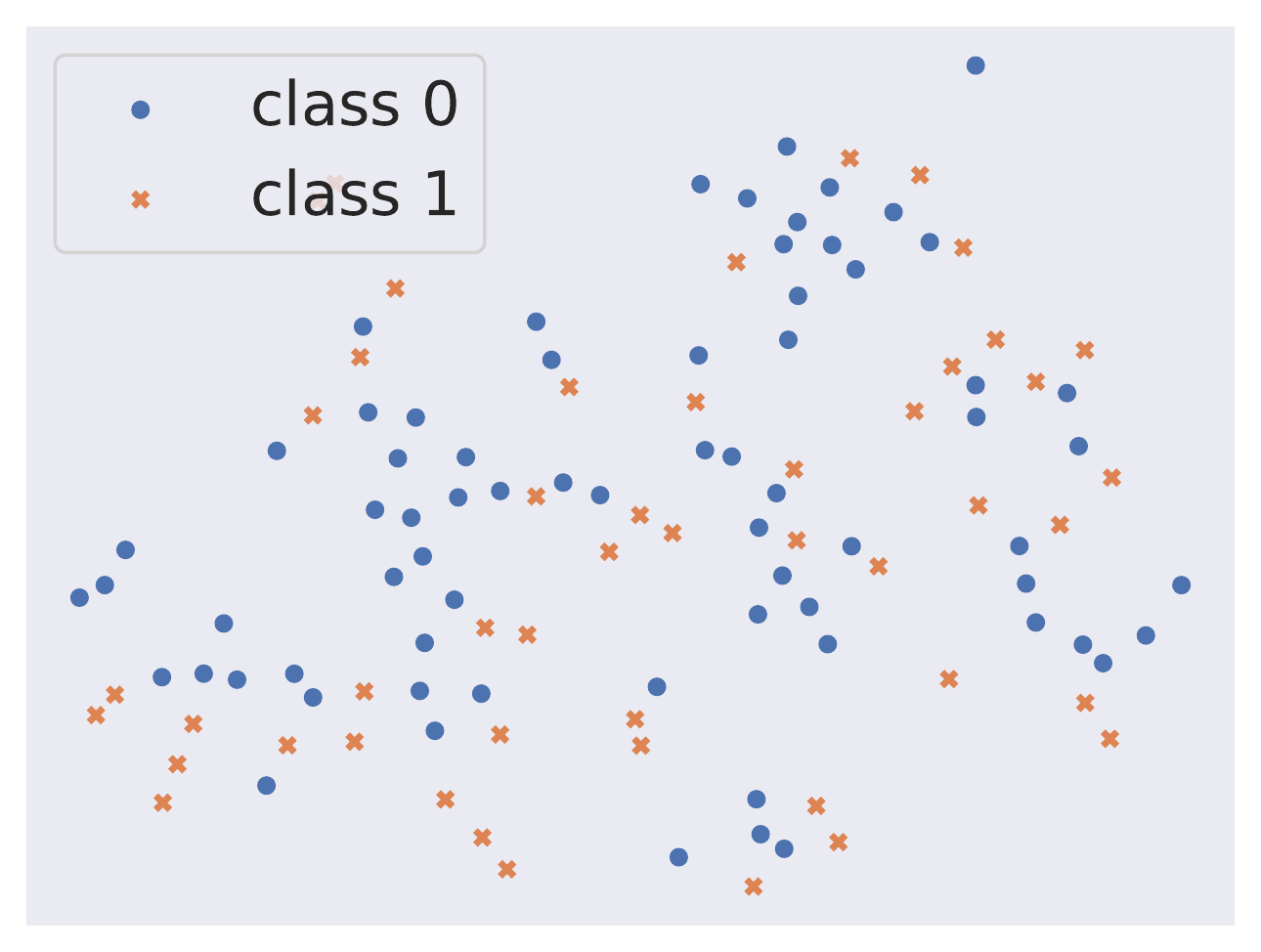} }}%
 \subfloat[\method]{\label{fig:one_step}{\includegraphics[width=0.21\linewidth]{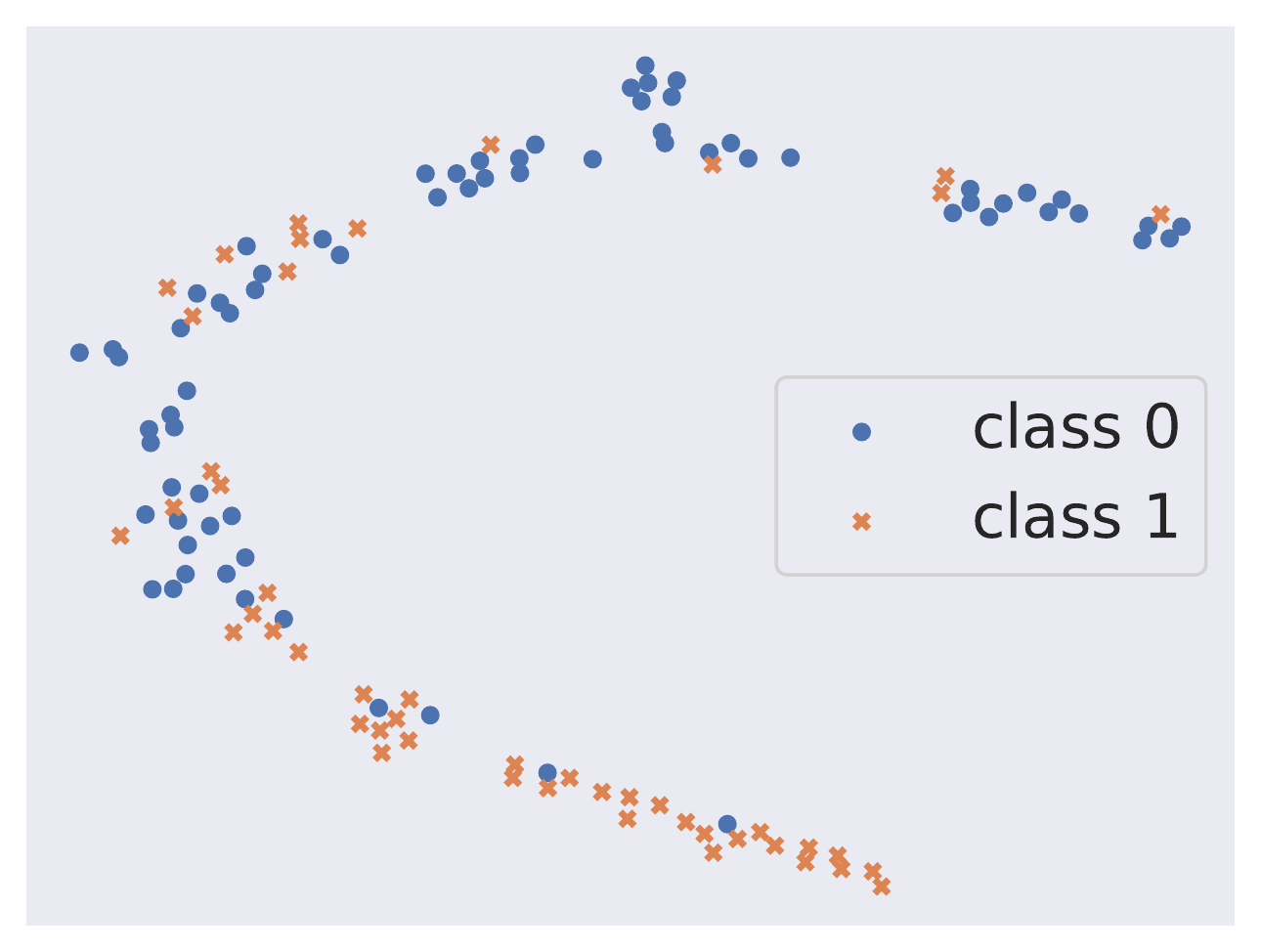} }}%
 \subfloat[Whole Dataset]{\label{fig:sparsity}{\includegraphics[width=0.21\linewidth]{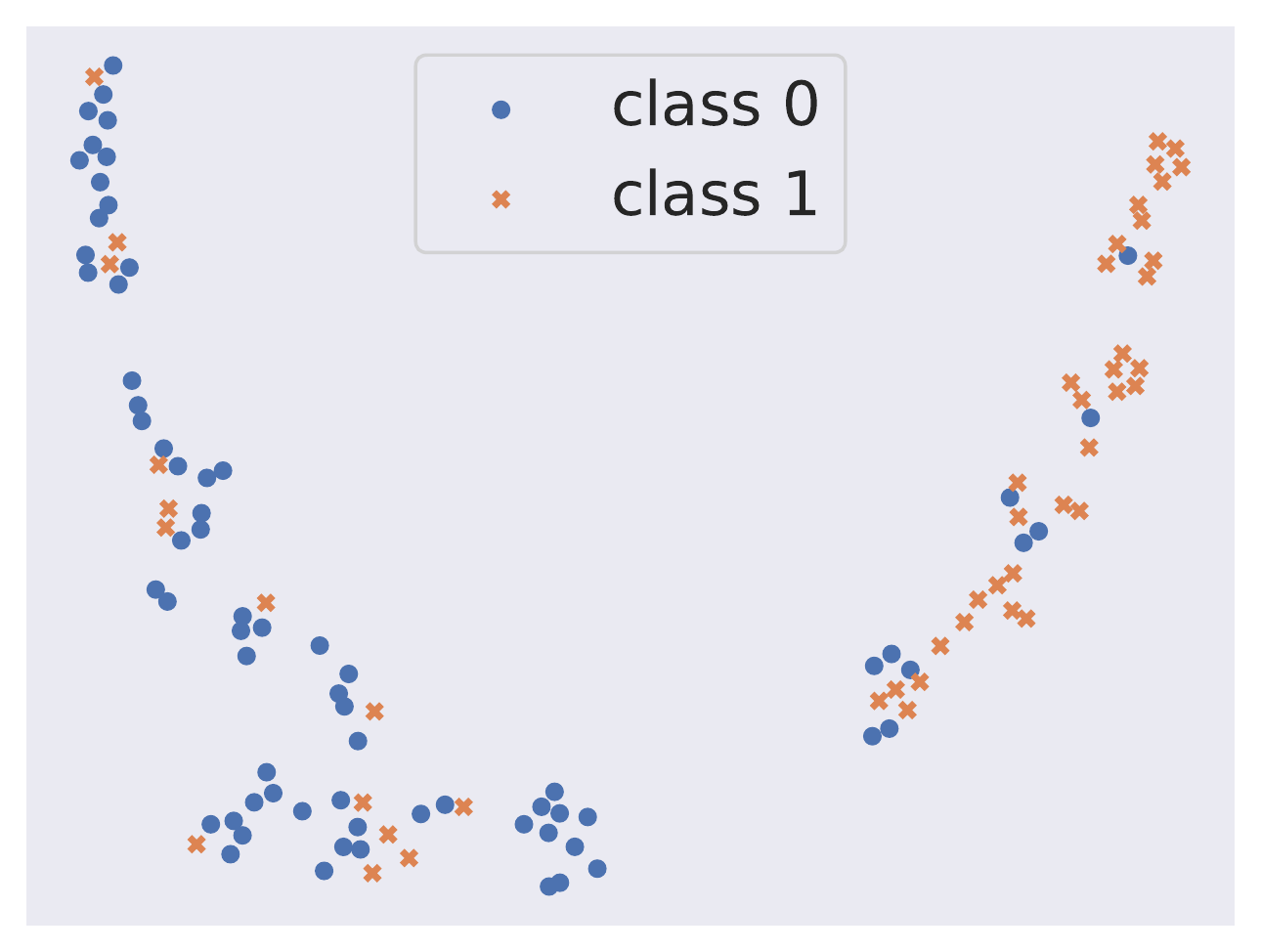} }}%
 \qquad
\vskip -1em 
\caption{T-SNE visualizations of embedding learned with condensed graphs on DD.}
\vskip -1em
\label{fig:vis}
\end{figure*}

\subsection{Further Investigation}
\label{sec:ablation}
In this subsection, we perform further investigations to provide a better understanding of our proposed method. 

\subsubsection{Increasing the Number of Synthetic Graphs.}\label{sec:increase}
We study whether the classification performance can be further boosted when using larger synthetic size. Concretely, we vary the size of the learned graphs from 1 to 300 and report the results of absolute and relative accuracy w.r.t. whole dataset training accuracy for CIFAR10 in Figure~\ref{fig:increase}. It is clear to see that both Random and \method{} achieve better performance when we increase the number of samples used for training. Moreover, our method outperforms the random baseline under different condensed dataset sizes. It is worth noting that the performance gap between the two methods diminishes with the increase of the number of samples. This is because the random baseline will finally approach the whole dataset training if we continue to enlarge the size of the condensed set, in which the performance can be considered as the upper bound of \method{}.

\subsubsection{Ablation Study.} To examine how different model components affect the model performance, we perform ablation study on the proposed one-step gradient matching and  regularization terms. We create an ablation of our method, namely \method{}\textit{-Bi}, which adopts the vanilla gradient matching scheme that involves a bi-level optimization. Without loss of generality, we compare the training time and classification accuracy of \method{} and \method{}\textit{-Bi} in the setting of learning 50 graphs/class synthetic graphs on CIFAR10 dataset. The results are summarized in Figure~\ref{fig:one_step} and we can see that \method{} needs approximately 5 minutes to reach the performance of \method{}\textit{-Bi} trained for 75 minutes, which indicates that \method{} only requires  6.7\% training cost. It further demonstrates the efficiency of  the proposed one-step gradient matching strategy.

Next we study the effect of sparsity regularization on \method{}. Specifically, we vary the sparsity coefficient $\beta$ in the range of $\{0, 0.001, 0.01, 0.1,$ $1,10\}$ and report the classification accuracy and graph sparsity on DD and NCI datasets in Figure~\ref{fig:sparsity}. Note that the graph sparsity is defined as the ratio of the number of edges to the square of the number of nodes. As shown in the figure, when $\beta$ gets larger, we exert a stronger regularization on the learned graphs and the graphs become more sparse. Furthermore, the increased sparsity does not affect  the classification performance. This is a desired property since sparse graphs can save much space for storage and reduce training cost for GNNs. We also remove the regularization of Eq.~\eqref{eq:reg_norm} for ogbg-molhiv, we obtain the performance of 0.724/ 0.727/0.731 for 1/10/50 graphs per class, which is slightly worse than the one with this regularization.

\begin{table*}[t]
\caption{Node classification accuracy (\%) comparison. The numbers in parentheses indicate the running time for 100 epochs and $r$ indicates the ratio of number of nodes in the condensed graph to that in the original graph.}
\vskip -1em
\label{tab:node}
\small
\begin{tabular}{@{}lccccc@{}}
\toprule
               & \multicolumn{1}{l}{\begin{tabular}[c]{@{}l@{}}Cora, $r$=2.6\%\end{tabular}} & \begin{tabular}[c]{@{}l@{}}Citeseer, $r$=1.8\%\end{tabular} & \multicolumn{1}{l}{\begin{tabular}[c]{@{}l@{}}Pubmed, $r$=0.3\%\end{tabular}} & \multicolumn{1}{l}{\begin{tabular}[c]{@{}l@{}}Arxiv, $r$=0.25\%\end{tabular}} & \multicolumn{1}{l}{\begin{tabular}[c]{@{}l@{}}Flickr, $r$=0.1\%\end{tabular}}  \\ \midrule
\textit{GCond}          & 80.1 (75.9s)                                                                       &  {70.6} (71.8s)                                   & 77.9 (51.7s)                                                                        & 59.2 (494.3s)                                                                           & 46.5 (51.9s)                                                                                                                 \\
\method{} & 80.0 (3.5s)                                                                       & 71.0 (2.8s)                                                          & 76.0 (1.3s)                                                                        & 59.0 (32.9s)                                                                               & 46.1 (14.3s)                                                                                                                                   \\ \midrule
Whole Dataset  & 81.5                                                                       &  {71.7}                                   & 79.3                                                                         & 71.4                                                                              & 47.2                                                                                                                    \\ \bottomrule
\end{tabular}
\vskip -1.2em
\end{table*}

\subsubsection{Visualization.} 
We further investigate whether GCN can learn discriminative representations from the synthetic graphs learned by \method{}. Specifically, we use t-SNE~\cite{van2008visualizing} to visualize the learned graph representation from GCN trained on different condensed graphs. We train a GCN on graphs produced by different methods and use it to extract the latent representation for real graphs from test set. Without loss of generality, we provide the t-SNE plots on DD dataset with 50 graphs per class in Figure~\ref{fig:vis}. It is observed that the graph representations learned with randomly selected graphs are mixed for different classes. This suggests that using randomly selected graphs cannot help GCN learn discriminative features. Similarly, DCG graphs also resulted in poorly trained GCN that outputs indistinguishable graph representations. By contrast, the representations are well separated for different classes when learned with \method{} graphs (Figure~\ref{fig:vis}c) and they are as discriminative as those learned on the whole training dataset (Figure~\ref{fig:vis}d). This demonstrates that the graphs learned by \method{} preserve sufficient information of the original dataset so as to recover the original performance.

\begin{figure}[t]%
\centering
  \subfloat[DD]{\label{fig:dd_loss}{\includegraphics[width=0.47\linewidth]{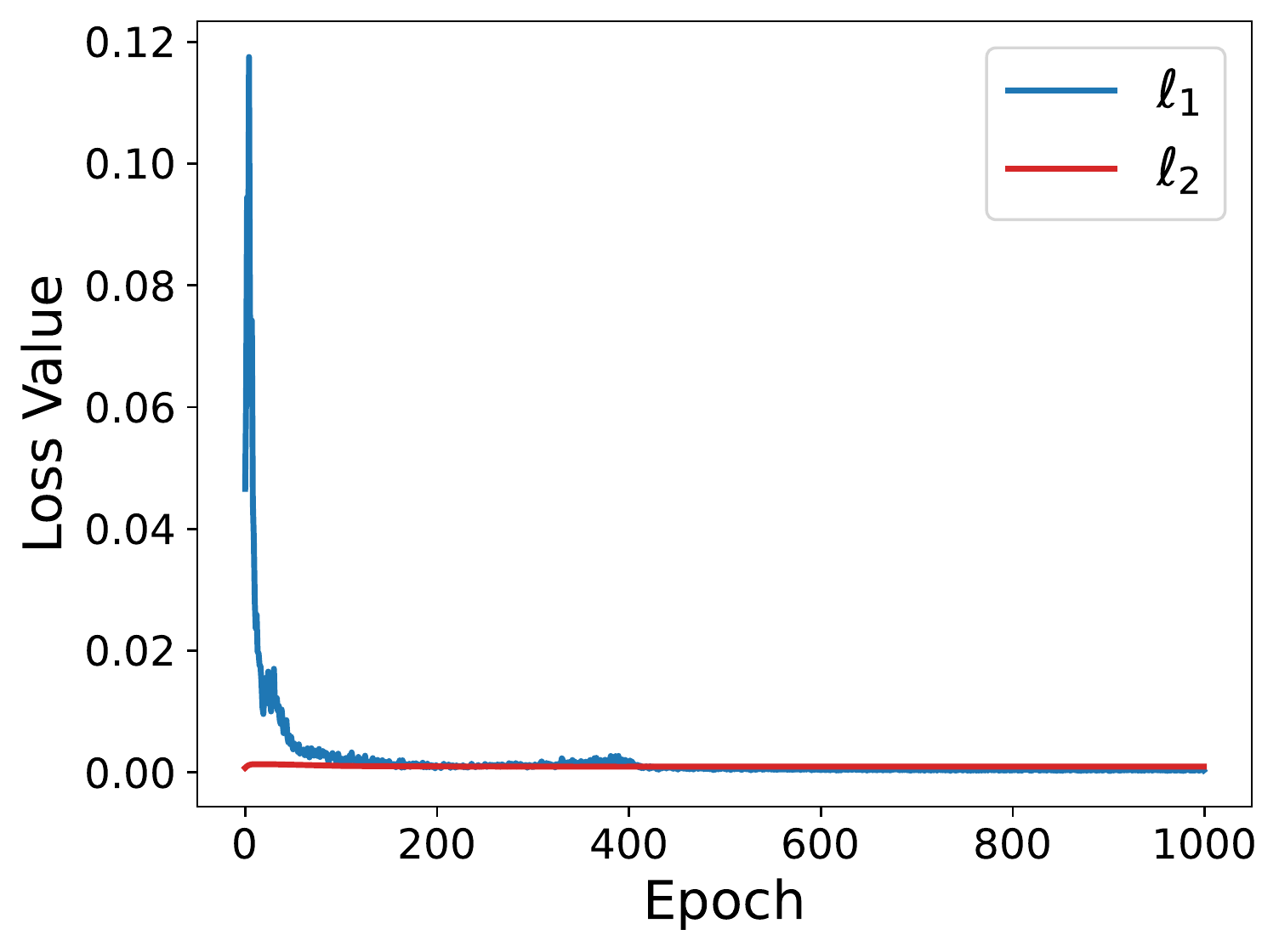} }}%
  \subfloat[ogbg-molhiv]{\label{fig:molhiv_loss}{\includegraphics[width=0.47\linewidth]{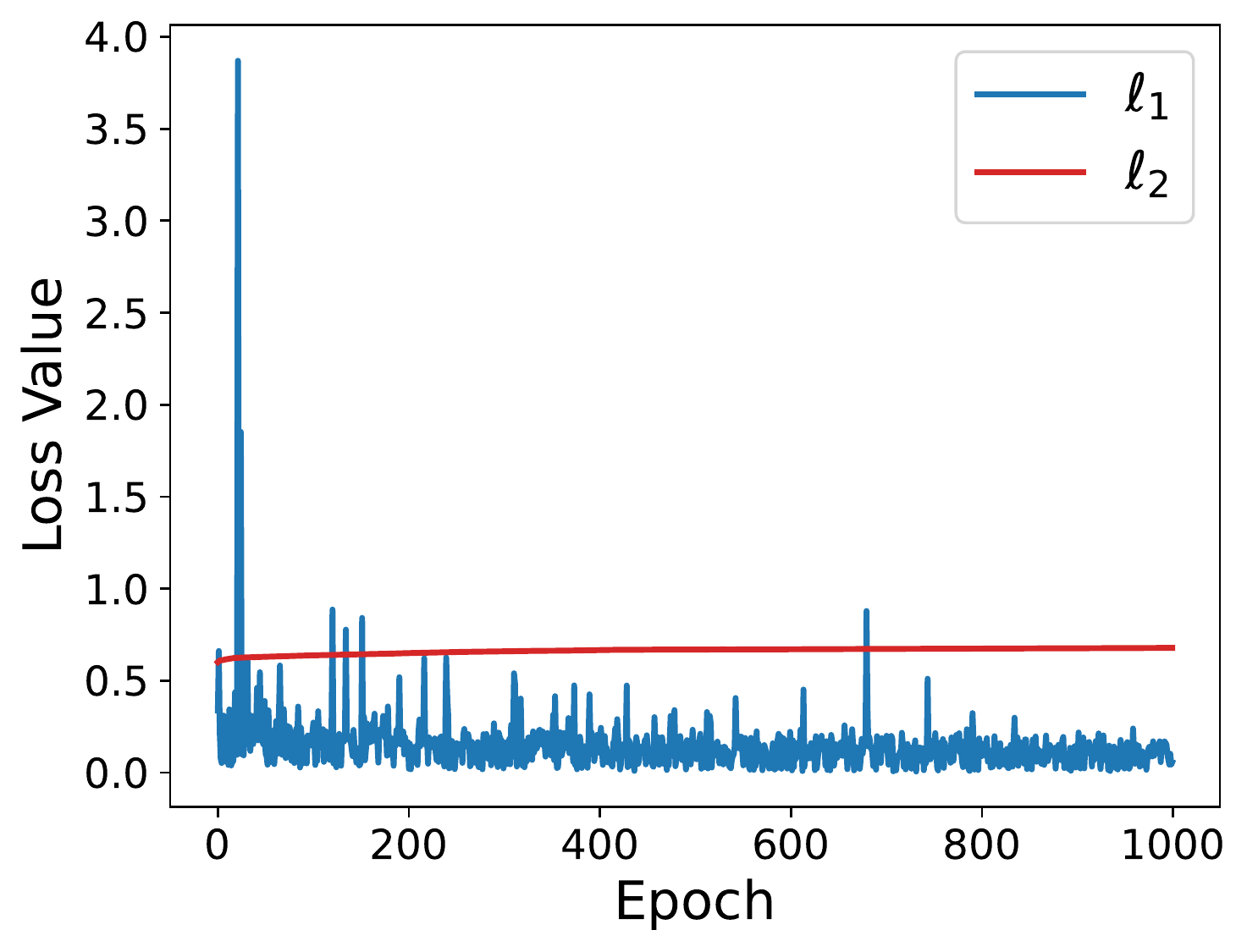} }}%
\vskip -1.2em
\caption{Scale of the two terms in Eq.~(11).}  
\label{fig:loss}
\vskip -2em
\end{figure}

\subsubsection{Scale of the two terms in Eq.~\eqref{eq:prop}.}\label{sec:terms}
As mentioned earlier in Section~\ref{sec:one_step}, the scale of the first term is essentially larger than the second term in Eq.~\eqref{eq:prop}. We now perform empirical study to verify this statement. Since both terms contain the factor $M$,  we simply drop it and focus on studying $\ell_1={\sqrt{2}}\|\nabla_{\theta} \ell_\mathcal{T}\left(\theta_{0}\right)- \nabla_{\theta} \ell_S\left(\theta_{0}\right)\| $ and $\ell_2=\frac{3}{2\sqrt{T}}\cdot\frac{C-1}{C N'}\sqrt{\sum_{i}\gamma_i\|{{\bf 1}^\top{\bf A}'^K_{(i)}}{\bf X}'_{(i)}\|^2}$. Specifically, we set $T$ to 500 and $N'$ to 50, and plot the changes of these two terms during the training process of \method{}. The results on DD (with mean pooling) and ogbg-molhiv (with sum pooling) are shown in Figure~\ref{fig:loss}. We can observe that the scale of $\ell_1$ is much larger than $\ell_2$ at the first few epochs when using mean pooling as shown in Figure~\ref{fig:dd_loss}. By contrast,  $\ell_2$ is not negligible  when using sum pooling as shown in Figure~\ref{fig:molhiv_loss} and it is desired to include it as a regularization term in this case. These observations provide support for ours discussion of theoretical analysis in Section~\ref{sec:one_step}.

\subsection{Node Classification}
\label{sec:node}
Next, we investigate whether the proposed method works well in node classification so as to support our analysis in Theorem~2 in Appendix~\ref{app:node}. Specifically, following \textit{GCond}~\cite{jin2022graph}, a condensation method for node classification, we use 5 node classification datasets: Cora, Citeseer, Pubmed~\cite{kipf2016semi}, ogbn-arxiv~\cite{hu2020open} and Flickr~\cite{graphsaint-iclr20}. The dataset statistics are shown in~\ref{tab:data_node}. We follow the settings in \textit{GCond} to generate one condensed graph for each dataset, train a GCN on the condensed graph, and evaluate its classification performance on the original test nodes. To adopt \method{} into node classification, we replace the bi-level gradient matching scheme in \textit{GCond} with our proposed one-step gradient matching. The results of classification accuracy and running time per epoch are summarized in Table~\ref{tab:node}.  
From the table, we make the following observations:
\begin{compactenum}[(a)]
\item The proposed \method{} achieves similar performance as \textit{GCond} and the performance is also comparable to the original dataset. For example, we are able to approximate the original training performance by 99\% with only 2.6\% data on Cora. It demonstrates the effectiveness of \method{} in the node classification case and justifies Theorem~2 from an empirical perspective. 
\item  The training cost of \method{} is essentially lower than \textit{GCond} as \method avoids the expensive bi-level optimization. By examining their running time, we can see that \method{} is up to 40 times faster than \textit{GCond}.
\end{compactenum}
We further note that \textit{GCond} produces weighted graphs which require storing the edge weights in float formats, while \method{} outputs discrete graph structure which can be stored as binary values. Hence, the graphs learned by \method{} are more memory-efficient.

\vskip -3em
\section{Related Work}\label{sec:relatedwork}

\noindent\textbf{Graph Neural Networks.} 
As the generalization of deep neural network to graph data, graph neural networks (GNNs)~\cite{kipf2016semi,klicpera2018predict-appnp,gat,wu2019comprehensive-survey,wu2019simplifying,tang2020transferring,jin2020graph,liu2022generating,wang2022improving} have revolutionized the field of graph representation learning through effectively exploiting graph structural information. GNNs have achieved remarkable performances in basic graph-related tasks such as graph classification~\cite{xu2018how,guo2021few}, link prediction~\cite{fan2019graph} and node classification~\cite{kipf2016semi}. Recent years have also witnessed their great success achieved in many real-world applications such as recommender systems~\cite{fan2019graph}, computer vision~\cite{li2019deepgcns}, drug discovery~\cite{duvenaud2015convolutional} and etc. GNNs take both adjacency matrix and node feature matrix as input and output node-level representations or graph-level representations. Essentially, they follow a message-passing scheme~\cite{gilmer2017neural} where each node  first aggregates the information from its neighborhood and then transforms the aggregated information to update its representation. Furthermore, there is significant progress in developing deeper GNNs~\cite{liu2020towards,jin2022feature}, self-supervised GNNs~\cite{you2021graph,you2020does,wang2022graph} and graph data augmentation~\cite{zhao2021data,ding2022data,zhao2022graph}.  

\vskip 0.1em
\noindent\textbf{Dataset Distillation \& Dataset Condensation.} 
It is widely received that training neural networks on large datasets can be prohibitively costly. To alleviate this issue, dataset distillation (DD)~\cite{wang2018dataset} aims to distill knowledge of a large training dataset into a small number of synthetic samples. DD formulates the distillation process as a learning-to-learning problem and solves it through bi-level optimization. 
To improve the efficiency of DD, dataset condensation (DC)~\cite{zhao2020dataset,zhao2021dataset} is proposed to learn the small synthetic dataset by matching the gradients of the network parameters w.r.t. large-real and small-synthetic training data. It has been demonstrated that these condensed samples can facilitate critical applications such as continual learning~\cite{zhao2020dataset,zhao2021dataset,kim2022dataset,lee2022dataset,zhao2021dataset-dm}, neural architecture search~\cite{nguyen2021dataset,nguyen2021infinitely,yang2022dataset} and privacy-preserving scenarios~\cite{dong2022privacy}
Recently, following the gradient matching scheme in DC, \citet{jin2022graph} propose a condensation method to condense a large-scale graph to a small graph for node classification. Different from \cite{jin2022graph} which learns weighted graph structure, we aim to solve the challenge of learning discrete structure and we majorly target at graph classification. Moreover, our method avoids the costly bi-level optimization and is much more efficient than the previous work. A detailed comparison is included in Section~\ref{sec:node}.



 \vspace{-0.5em}
\section{Conclusion}\label{sec:conclusion}
Training graph neural networks on a large-scale graph dataset consumes high computational cost. One solution to alleviate this issue is to condense the large graph dataset into a small synthetic dataset. In this work, we propose a novel framework \method{} that adopts a one-step gradient matching strategy to efficiently condenses real graphs into a small number of informative graphs with discrete structures. We further justify the proposed method from both theoretical and empirical perspectives. Notably, our experiments show that we are able to reduce the dataset size by 90\% while approximating up to 98\% of the original performance. In the future, we plan to investigate interpretable  condensation methods and diverse applications of the condensed graphs.

\vskip -1.5em
\section*{ACKNOWLEDGEMENT}
Wei Jin and Jiliang Tang are supported by  the National Science Foundation (NSF) under grant numbers IIS1714741, CNS1815636, IIS1845081, IIS1907704, IIS1928278, IIS1955285, IOS2107215, and  \allowbreak IOS2035472, the Army Research Office (ARO) under grant number W911NF-21-1-0198, and Amazon.com, Inc.


\balance

\bibliographystyle{ACM-Reference-Format}
\bibliography{sample}

\appendix

\section{Experimental Setup}
\subsection{Algorithm}
\label{app:alg}
\begin{algorithm}[!h]
\small
\caption{\method{} for Condensing Graphs}
\begin{algorithmic}[1]
\label{alg}
\STATE {\bf Input:} Training data $\mathcal{T}=({\bf A}, {\bf X}, \mathcal{Y})$
\STATE {\bf Required:} Pre-defined condensed labels $\mathcal{Y}'$, graph neural network $f_\theta$, temperature $\tau$,  desired sparsity $\epsilon$, regularization coefficient $\beta$, learning rates $\eta_1, \eta_2$, number of epochs $K_1$.
\STATE Initialize ${{\bf {\bf \Omega}}}, {\bf X'}$ 
\FOR{$k=0,\ldots, K_1-1$}
  \STATE Sample ${\theta}_0\sim P_{{\theta}_0}$\\
\STATE Sample $\alpha\sim\operatorname{Uniform}(0,1)$
  \STATE
  Compute ${\bf A'} = \sigma\left(\left(\log \alpha-\log (1-\alpha)+{\bf \Omega}\right) / \tau\right)$ 
  \FOR{$c=0,\ldots,C-1$}  
  \STATE Sample $({\bf A}_c,{\bf X}_c, \mathcal{Y}_c)\sim \mathcal{T}$ and $({\bf A}_{c}',{\bf X}_{c}', \mathcal{Y}_{c}')\sim \mathcal{S}$ \\
  \STATE Compute $\ell_{{T}}=\ell\left(f_{{\theta}_{0}}({\bf A}_c,{\bf X}_c), \mathcal{Y}_c\right)$ 
  \STATE Compute $\ell_{{S}}= \ell\left(f_{{\theta}_{0}}({\bf A}_{c}',{\bf X}_{c}'), \mathcal{Y}_{c}'\right)$ \\
  \STATE Compute $\ell_\text{reg} = \operatorname{max}(\sum_{i,j}\sigma({\bf \Omega}_{ij}) - \epsilon, 0)$
  \STATE  Update ${{\bf {\bf \Omega}}} \leftarrow {\bf \Omega} -\eta_1 \nabla_{\bf \Omega}    (D(\nabla_{\boldsymbol{\theta}_0}\ell_{{T}}, \nabla_{\boldsymbol{\theta}_0}\ell_{{S}})+\beta\ell_\text{reg})$\\
  \STATE  Update ${\bf X'} \leftarrow {\bf X'} -\eta_2 \nabla_{\bf X'}    (D(\nabla_{\boldsymbol{\theta}_0}\ell_{{T}}, \nabla_{\boldsymbol{\theta}_0}\ell_{{S}})+\beta\ell_\text{reg})$\\
  \ENDFOR
 \ENDFOR
\STATE \textbf{Return:} $({\bf \Omega}, {\bf X'}, \mathcal{Y}')$ \\
\end{algorithmic}
\end{algorithm}

\subsection{Dataset Statistics and Code}
\label{appendix:data}
Dataset statistics are shown in Table 4 and 5.  We provide our code in the supplementary file for the purpose of reproducibility.

\begin{table}[h!]
\vskip -1em
\fontsize{8.7}{9.2}\selectfont
 \setlength{\tabcolsep}{1pt}
    \caption{Graph classification dataset
    statistics.}\label{tab:data-semantics}
    \vskip -1em
    \centering
    \begin{tabular}{llcccc}
    \toprule
    Dataset &  Type & \#Clases & \#Graphs & Avg. Nodes & Avg. Edges \\
    \midrule  
    CIFAR10   &  Superpixel        & 10 & 60,000 & 117.6 & 941.07 \\
    ogbg-molhiv    & Molecule & 2  & 41,127  & 25.5 & 54.9  \\
    ogbg-molbace & Molecule& 2&1,513 & 34.1 & 36.9  \\
    ogbg-molbbbp & Molecule& 2 &2,039 &24.1 &26.0 \\
    MUTAG     & Molecule & 2 & 188 & 	17.93 & 19.79 \\
    NCI1      & Molecule  & 2 & 4,110 & 29.87 & 32.30  \\
    DD  & Molecule & 2 & 1,178 & 284.32 & 715.66  \\
    E-commerce & Transaction & 2 & 1,109 & 33.7 & 56.3 \\ 
    \bottomrule
    \end{tabular}
\end{table}

\begin{table}[h]
\small
\vskip -1em
\caption{Node classification dataset statistics.}
\vskip -1em
\label{tab:data_node}
\begin{tabular}{@{}lcccc@{}}
\toprule
Dataset  & \#Nodes & \#Edges   & \#Classes & \#Features \\ \midrule
Cora     & 2,708   & 5,429     & 7         & 1,433      \\
Citeseer & 3,327   & 4,732     & 6         & 3,703      \\
Pubmed   & 19,717  & 44,338    & 3         & 500        \\
Arxiv    & 169,343 & 1,166,243 & 40        & 128        \\
Flickr   & 89,250  & 899,756   & 7         & 500        \\ \bottomrule
\end{tabular}
\end{table}




\section{Proofs}
\setcounter{theorem}{0}

\subsection{Proof of Theorem~\ref{thm:graph}}
\label{appendix:thm1}

Let ${\bf A}_{(i)}$, ${\bf X}_{(i)}$ denote the adjacency matrix and  the feature matrix of $i$-th real graph, respectively. We denote the cross entropy loss on the real samples as $\ell_\mathcal{T}(\theta)=\sum_{i}\ell_i({\bf A}_{(i)}$, ${\bf X}_{(i)}, \theta)$ and denote that on synthetic samples as $\ell_{S}(\theta)=\ell_{S}({\bf A}'_{(i)}, {\bf X}'_{(i)}, \theta)$.
Let $\theta^*$ denote the optimal parameter and let $\theta_t$ be the parameter trained on condensed data at $t$-th epoch by optimizing $\ell_{S}(\theta)$. 
For simplicity of notations, we assume ${\bf A}$ and ${\bf A}'$ are already normalized.  Part of the proof is inspired from the work~\cite{pmlr-v139-killamsetty21a}.


\begin{theorem}
\label{thm:graph}
When we use a linearized $K$-layer SGC as the GNN used in condensation, i.e., $f_\theta({\bf A}_{(i)}, {\bf X}_{(i)})= \text{Pool}({\bf A}_{(i)}^K{\bf X}_{(i)}{\bf W}_1){\bf W}_2$ with $\theta=[{\bf W}_1;{\bf W}_2]$ and assume that all network parameters satisfy $\|\theta\|^2\leq M^2 (M>0)$, we have 
\begin{align}
\min _{t=0,1, \ldots, T-1}  & \ell_\mathcal{T}\left(\theta_{t}\right)-\ell_\mathcal{T}\left(\theta^{*}\right) \leq   \sum_{t=0}^{T-1} \frac{\sqrt{2}M}{T} \|\nabla_{\theta} \ell_\mathcal{T}\left(\theta_{t}\right)- \nabla_{\theta} \ell_{S}\left(\theta_{t}\right)\| 
\nonumber \\
& +  \frac{3M}{2\sqrt{T}}\cdot\frac{C-1}{C N'}\sqrt{\sum_{i}\gamma_i\|{{\bf 1}^\top{\bf A}'^K_{(i)}}{\bf X}'_{(i)}\|^2}
\end{align}
where $\gamma_i=1$ if we use sum pooling in $f_\theta$;  $\gamma_i=\frac{1}{n_i}$ if we use mean pooling, with $n_i$ being the number of nodes in $i$-th synthetic graph. 
\end{theorem}

\begin{proof}
We  start by proving that $\ell_\mathcal{T}(\theta)$ is convex and $\ell_{S}(\theta)$ is lipschitz continuous when we use $f_\theta({\bf A}_{(i)}, {\bf X}_{(i)})=\text{Pool}({\bf A}_{(i)}^K{\bf X}_{(i)}{\bf W}_1){\bf W}_2$ as the mapping function. Before  proving these two properties, we first rewrite $f_\theta({\bf A}_{(i)}, {\bf X}_{(i)})$ as:
\begin{equation}
f_\theta({\bf A}_{(i)}, {\bf X}_{(i)})= 
 \begin{cases}
{\bf 1}^\top{\bf A}_{(i)}^K{\bf X}_{(i)}{\bf W}_1{\bf W}_2 &  \text{if use sum pooling}, \\
\frac{1}{n_i}{\bf 1}^\top{\bf A}_{(i)}^K{\bf X}_{(i)}{\bf W}_1{\bf W}_2 & \text{if use mean pooling},
\end{cases}
\end{equation}
where $n$ is the number of nodes in ${\bf A}_{(i)}$  and ${\bf 1}$ is an ${n_i}\times{1}$ matrix filled with constant one. From the above equation we can see that $f_\theta$ with different pooling methods only differ in a multiplication factor $\frac{1}{n_i}$. Thus, in the following we focus on $f_\theta$ with sum pooling to derive the major proof. 

\noindent\textbf{I. For $f_\theta$ with sum pooling:}

\noindent Substitute ${\bf W}$ for ${\bf W}_1{\bf W}_2$ and we have $f_\theta({\bf A}_{(i)}, {\bf X}_{(i)})={\bf 1}^\top{\bf A}_{(i)}^K{\bf X}_{(i)}{\bf W}$ for the case with sum pooling. Next we show that $\ell_\mathcal{T}(\theta)$ is convex and $\ell_{S}(\theta)$ is lipschitz continuous when we use $f_\theta({\bf A}_{(i)}, {\bf X}_{(i)})={\bf 1}^\top{\bf A}_{(i)}^K{\bf X}_{(i)}{\bf W}$ with $\theta={\bf W}$.

\noindent(a) Convexity of $\ell_\mathcal{T}(\theta)$.  From chapter 4 of the book~\cite{watt2020machine}, we know that softmax classification $f({\bf W})={\bf XW}$ with cross entropy loss is convex w.r.t. the parameters ${\bf W}$. In our case, the mapping function $f_\theta({\bf A}_{(i)}, {\bf X}_{(i)})={\bf 1}^\top{\bf A}_{(i)}^K{\bf X}_{(i)}{\bf W}$ applies an affine function on ${\bf XW}$. Given that applying affine function does not change the convexity, we know that $\ell_\mathcal{T}(\theta)$ is convex.



\noindent(b) Lipschitz continuity of $\ell_{S}(\theta)$. 
In~\cite{yedida2021lipschitzlr}, it shows that the lipschitz constant of softmax regression with cross entropy loss is $\frac{C-1}{C m}\|\mathbf{X}\| $, where ${\bf X}$ is the input feature matrix, $C$ is the number of classes and $m$ is the number of samples. Since $\ell_{S}(\theta)$ is cross entropy loss and $f_\theta$ is linear, we know that the $f_\theta$ is lipschitz continuous and it satisfies:
\begin{equation}
    \nabla_\theta\ell_{S}(\theta)\leq \frac{C-1}{C N'} \sqrt{\sum_{i}\|{{\bf 1}^\top{\bf A}'^K_{(i)}}{\bf X}'_{(i)}\|^2}
\label{eq:graph_lip}
\end{equation}

With (a) and (b), we are able to proceed our proof. First, from the convexity of $\ell_\mathcal{T}(\theta)$ we have \begin{equation}
\ell_\mathcal{T}\left(\theta_{t}\right)-\ell_\mathcal{T}\left(\theta^{*}\right) \leq \nabla_{\theta} \ell_\mathcal{T}\left(\theta_{t}\right)^{T}\left(\theta_{t}-\theta^{*}\right)
\label{eq:graph_convex}
\end{equation}

We can rewrite $\nabla_{\theta} \ell_\mathcal{T}\left(\theta_{t}\right)^{T}\left(\theta_{t}-\theta^{*}\right)$ as follows:
\begin{small}
\begin{align}
&\nabla_{\theta} \ell_\mathcal{T}\left(\theta_{t}\right)^{T}\left(\theta_{t}-\theta^{*}\right)  =    (\nabla_{\theta} \ell_\mathcal{T}\left(\theta_{t}\right)^{T}- \nabla_{\theta} \ell_{S}\left(\theta_{t}\right)^{T} + \nabla_{\theta} \ell_{S}\left(\theta_{t}\right)^{T}) \left(\theta_{t}-\theta^{*}\right) \nonumber \\ 
& =  (\nabla_{\theta} \ell_\mathcal{T}\left(\theta_{t}\right)^{T}- \nabla_{\theta} \ell_{S}\left(\theta_{t}\right)^{T}) \left(\theta_{t}-\theta^{*}\right)  + \nabla_{\theta} \ell_{S}\left(\theta_{t}\right)^{T} \left(\theta_{t}-\theta^{*}\right) 
\end{align}
\end{small}

Given that we use gradient descent to update network parameters, we have $\nabla_{\theta} \ell_{S}\left(\theta_{t}\right) = \frac{1}{\eta}{(\theta_{t}-\theta_{t+1})}$ where ${\eta}$ is the learning rate. Then we have,
\begin{align}
\nabla_{\theta} \ell_{S}&
\left(\theta_{t}\right)^{T}  \left(\theta_{t}- \theta^{*}\right) =\frac{1}{\eta}\left(\theta_{t}-\theta_{t+1}\right)^{T}\left(\theta_{t}-\theta^{*}\right) \nonumber \\
& = \frac{1}{2 \eta}\left(\left\|\theta_{t}-\theta_{t+1}\right\|^{2}+\left\|\theta_{t}-\theta^{*}\right\|^{2}-\left\|\theta_{t+1}-\theta^{*}\right\|^{2}\right) \nonumber \\
& = \frac{1}{2 \eta}\left(\left\|\eta \nabla_{\theta} \ell_{S}\left(\theta_{t}\right)\right\|^{2}+\left\|\theta_{t}-\theta^{*}\right\|^{2}-\left\|\theta_{t+1}-\theta^{*}\right\|^{2}\right)
\label{eq:graph_gd}
\end{align}
Combining Eq.~\eqref{eq:graph_convex} and Eq.~\eqref{eq:graph_gd} we have,
\begin{small}
\begin{align}
\ell_\mathcal{T}  \left(\theta_{t}\right)-&\ell_\mathcal{T}\left(\theta^{*}\right) \leq  (\nabla_{\theta} \ell_\mathcal{T}\left(\theta_{t}\right)^{T}- \nabla_{\theta} \ell_{S}\left(\theta_{t}\right)^{T}) \left(\theta_{t}-\theta^{*}\right) \nonumber \\  
& + \frac{1}{2 \eta}\left(\left\|\eta \nabla_{\theta} \ell_{S}\left(\theta_{t}\right)\right\|^{2}+\left\|\theta_{t}-\theta^{*}\right\|^{2}-\left\|\theta_{t+1}-\theta^{*}\right\|^{2}\right)
\end{align}
\end{small}
We sum up the two sides of the above inequality for different values of $t \in[0, T-1]$:
\begin{align}
 \sum_{t=0}^{T-1}  \ell_\mathcal{T}& \left(\theta_{t}\right)-  \ell_\mathcal{T}\left(\theta^{*}\right) \leq   \sum_{t=0}^{T-1}  (\nabla_{\theta} \ell_\mathcal{T}\left(\theta_{t}\right)^{T}- \nabla_{\theta} \ell_{S}\left(\theta_{t}\right)^{T}) \left(\theta_{t}-\theta^{*}\right)   \nonumber \\
 & + \frac{1}{2 \eta}\sum_{t=0}^{T-1} \left\|\eta \nabla_{\theta} \ell_{S}\left(\theta_{t}\right)\right\|^{2}
 +\frac{1}{2 \eta}\left\|\theta_{0}-\theta^{*}\right\|^{2}-\frac{1}{2 \eta}\left\|\theta_{T}-\theta^{*}\right\|^{2}  
\end{align}
Since $\frac{1}{2 \eta}\left\|\theta_{T}-\theta^{*}\right\|^{2} \geq 0$, we have
\begin{align}
 \sum_{t=0}^{T-1}  \ell_\mathcal{T}\left(\theta_{t}\right)-& \ell_\mathcal{T}\left(\theta^{*}\right) \leq   \sum_{t=0}^{T-1}  (\nabla_{\theta} \ell_\mathcal{T}\left(\theta_{t}\right)^{T}- \nabla_{\theta} \ell_{S}\left(\theta_{t}\right)^{T}) \left(\theta_{t}-\theta^{*}\right)  \nonumber \\
 &+ \frac{1}{2 \eta}\sum_{t=0}^{T-1} \left\|\eta \nabla_{\theta} \ell_{S}\left(\theta_{t}\right)\right\|^{2}
 +\frac{1}{2 \eta}\left\|\theta_{0}-\theta^{*}\right\|^{2} 
 \label{eq:bound1}
\end{align}
As we assume that $\left\|\theta\right\|^2 \leq M^2$, we have $\|\theta-\theta^{*}\|^2\leq 2\|\theta\|^2=2M^2$. Then Eq.~\eqref{eq:bound1} can be rewritten as,
\begin{align}
 \sum_{t=0}^{T-1}  \ell_T\left(\theta_{t}\right)
 -\ell_T\left(\theta^{*}\right) & \leq   \sum_{t=0}^{T-1} \sqrt{2}M \|\nabla_{\theta} \ell_T\left(\theta_{t}\right)- \nabla_{\theta} \ell_{S}\left(\theta_{t}\right)\|  \nonumber \\
 & + \frac{1}{2 \eta}\sum_{t=0}^{T-1} \left\|\eta \nabla_{\theta} \ell_{S}\left(\theta_{t}\right)\right\|^{2}
 +\frac{M^2}{\eta}
 \label{eq:graph_convex_final}
\end{align}
Recall that $\ell_{S}(\theta)$ is lipschitz continuous as shown in Eq.~\eqref{eq:graph_lip},
and combine $\min\limits_{t=0,1, \ldots, T-1} \left(\ell_\mathcal{T}\left(\theta_{t}\right)-\ell_\mathcal{T}\left(\theta^{*}\right)\right) \leq \frac{\sum_{t=0}^{T-1} \ell_\mathcal{T}\left(\theta_{t}\right)-\ell_\mathcal{T}\left(\theta^{*}\right)}{T}$:
\begin{align}
\min _{t=0,1, \ldots, T-1}   \ell_\mathcal{T}\left(\theta_{t}\right)- &
\ell_\mathcal{T}\left(\theta^{*}\right) \leq   \sum_{t=0}^{T-1} \frac{\sqrt{2}M}{T} \|\nabla_{\theta} \ell_\mathcal{T}\left(\theta_{t}\right)- \nabla_{\theta} \ell_{S}\left(\theta_{t}\right)\|  \nonumber \\
 & +  \frac{\eta(C-1)^2}{2C^2 {N'}^2}{\sum_{i}\|{{\bf 1}^\top{\bf A}'^K_{(i)}}{\bf X}'_{(i)}\|^2}
 +\frac{M^2}{T\eta}
\end{align}
Then we choose $\eta=\frac{M}{\sqrt{T}\sqrt{{\sum_{i}\|{{\bf 1}^\top{\bf A}'^K_{(i)}}{\bf X}'_{(i)}\|^2}}}$ and we can get:
\begin{align}
\min _{t=0,1, \ldots, T-1}   \ell_\mathcal{T}\left(\theta_{t}\right)- &
\ell_\mathcal{T}\left(\theta^{*}\right) \leq   \sum_{t=0}^{T-1} \frac{\sqrt{2}M}{T} \|\nabla_{\theta} \ell_\mathcal{T}\left(\theta_{t}\right)- \nabla_{\theta}  \ell_{S}\left(\theta_{t}\right)\|  \nonumber \\
 & +  \frac{3M}{2\sqrt{T}}\cdot\frac{C-1}{C N'}\sqrt{\sum_{i}\|{{\bf 1}^\top{\bf A}'^K_{(i)}}{\bf X}'_{(i)}\|^2}
\end{align}

\noindent\textbf{II. For $f_\theta$ with mean pooling:}

\noindent Following similar derivation as in the case of sum pooling, we have
\begin{align}
\min _{t=0,1, \ldots, T-1}   \ell_\mathcal{T}\left(\theta_{t}\right)- &
\ell_\mathcal{T}\left(\theta^{*}\right) \leq   \sum_{t=0}^{T-1} \frac{\sqrt{2}M}{T} \|\nabla_{\theta} \ell_\mathcal{T}\left(\theta_{t}\right)- \nabla_{\theta} \ell_{S}\left(\theta_{t}\right)\|  \nonumber \\
& +  \frac{3M}{2\sqrt{T}}\cdot\frac{C-1}{C N'}\sqrt{\sum_{i}\frac{1}{n_i}\|{{\bf 1}^\top{\bf A}'^K_{(i)}}{\bf X}'_{(i)}\|^2}
\end{align}
where $n_i$ is the number of nodes in $i$-th synthetic graph.
\end{proof}

\subsection{Theorem for Node Classification Case}
\label{app:node}
\setcounter{theorem}{1}
We adopt similar notations for representing the data in node classification but note that there is only one graph for node classification task. Let ${\bf A}\in\{0,1\}^{N\times{N}}$, ${\bf A}'\in\{0,1\}^{N'\times{N'}}$  denote the adjacency matrix for real graph and synthetic graph, respectively. Let ${\bf X}\in\mathbb{R}^{N\times{d}}$, ${\bf X}'\in\mathbb{R}^{N'\times{d}}$ denote the feature matrix for real graph and synthetic graph, respectively. 
We denote the cross entropy loss on the real samples as $\ell_\mathcal{T}(\theta)$ and denote that on synthetic samples as $\ell_{S}(\theta)$.

\begin{theorem}
When we use a $K$-layer SGC as the model used in condensation, i.e., $f_\theta({\bf A}, {\bf X}, \theta)={{\bf A}}^K{\bf X}{\bf W}$ with $\theta={\bf W}$ and assume that all network parameters satisfy $\|\theta\|^2\leq M^2 (M>0)$, we have 
\begin{align}
\min _{t=0,1, \ldots, T-1}   \ell_\mathcal{T}\left(\theta_{t}\right)- & 
\ell_\mathcal{T}\left(\theta^{*}\right) \leq   \sum_{t=0}^{T-1} \frac{\sqrt{2}M}{T} \|\nabla_{\theta} \ell_\mathcal{T}\left(\theta_{t}\right)- \nabla_{\theta} \ell_{S}\left(\theta_{t}\right)\| \nonumber \\
& +  \frac{3M}{2\sqrt{T}}\cdot\frac{C-1}{C N'}\|{{\bf A}'}^K{\bf X}'\| 
\end{align}
\end{theorem}

\begin{proof}
We start by proving that $\ell_\mathcal{T}(\theta)$ is convex and $\ell_{S}(\theta)$ is lipschitz continuous when $f_\theta({\bf A}, {\bf X}, \theta)={{\bf A}}^K{\bf X}{\bf W}$.

\noindent(a) Convexity of $\ell_\mathcal{T}(\theta)$: Similar to the graph classification case, the Hessian matrix of $\ell_\mathcal{T}(\theta)$ in node classification is positive semidefinite and thus $\ell_\mathcal{T}(\theta)$ is convex. 

\noindent(b) Lipschitz continuity of $\ell_{S}(\theta)$: As shown in~\cite{yedida2021lipschitzlr},  the lipschitz constant of softmax regression with cross entropy loss is $\frac{C-1}{C m}\|\mathbf{X}\| $ with $C$ being the number of classes and $m$ being the number of samples. Thus, we know that the lipschitz constant of $\ell_{S}(\theta)$ is $\frac{C-1}{C N'}\|{{\bf A}'}^K{\bf X}'\| $, which indicates $\nabla_\theta\ell_{S}(\theta)\leq \frac{C-1}{C N'}\|{{\bf A}'}^K{\bf X}'\| $. 

From the convexity of $\ell_\mathcal{T}(\theta)$, we still have the following inequality (see Eq.~\eqref{eq:graph_convex_final}).
Then recall that $\ell_{S}(\theta)$ is lipschitz continuous and  $\nabla_\theta\ell_{S}(\theta)\leq \frac{C-1}{C N'}\|{{\bf A}'}^K{\bf X}'\| $, and
combine $\min\limits_{t} \left(\ell_\mathcal{T}\left(\theta_{t}\right)-\ell_\mathcal{T}\left(\theta^{*}\right)\right) \leq \frac{\sum_{t=0}^{T-1} \ell_\mathcal{T}\left(\theta_{t}\right)-\ell_\mathcal{T}\left(\theta^{*}\right)}{T}$:
\begin{small}
\begin{align}
\min _{t=0,1, \ldots, T-1}   \ell_\mathcal{T}\left(\theta_{t}\right)-\ell_\mathcal{T}\left(\theta^{*}\right) \leq  &
\sum_{t=0}^{T-1} \frac{\sqrt{2}M}{T} \|\nabla_{\theta} \ell_\mathcal{T}\left(\theta_{t}\right)- \nabla_{\theta} \ell_{S}\left(\theta_{t}\right)\|  \nonumber \\
& +  \frac{\eta(C-1)^2}{2C^2 {N'}^2}\|{{\bf A}'}^K{\bf X}'\| ^2
 +\frac{M^2}{T\eta}
\end{align}
\end{small}
Then we choose $\eta=\frac{M}{  \sqrt{T}\|{{\bf A}'}^K{\bf X}'\|}$ and we can get:
\begin{small}
\begin{align}
\min _{t=0,1, \ldots, T-1}   \ell_\mathcal{T}\left(\theta_{t}\right)-
\ell_\mathcal{T}\left(\theta^{*}\right) \leq  & \sum_{t=0}^{T-1} \frac{\sqrt{2}M}{T} \|\nabla_{\theta} \ell_\mathcal{T}\left(\theta_{t}\right)- \nabla_{\theta} \ell_{S}\left(\theta_{t}\right)\| \nonumber \\
& +  \frac{3M}{2\sqrt{T}}\cdot\frac{C-1}{C N'}\|{{\bf A}'}^K{\bf X}'\| 
\end{align}
\end{small}
\end{proof}


\end{document}